\documentclass[twoside]{article}

\usepackage[accepted]{aistats2025}
\usepackage{hyperref}
\usepackage{url}
\usepackage[utf8]{inputenc} 
\usepackage[T1]{fontenc}    
\usepackage{hyperref}       
\usepackage{url}            
\usepackage{booktabs}       
\usepackage{amsfonts}       
\usepackage{nicefrac}       
\usepackage{microtype}      
\usepackage{xcolor}         
\usepackage{mathtools}

\usepackage{siunitx}
\usepackage{physics}
\usepackage{cancel}
\usepackage{algorithm}
\usepackage{algpseudocode}
\usepackage{amsmath}
\usepackage{tikz}

\newcommand{\bv}{\mathbf{v}}
\newcommand{\bw}{\mathbf{w}}
\newcommand{\bx}{\mathbf{x}}
\newcommand{\by}{\mathbf{y}}
\newcommand{\bz}{\mathbf{z}}

\newcommand{\bQ}{\mathbf{Q}}

\usepackage{amsthm}
\usepackage{natbib}
\usepackage{mathdots}
\usepackage{amssymb}

\DeclareMathOperator*{\argmin}{arg\,min}
\DeclareMathOperator*{\argmax}{arg\,max}

\theoremstyle{definition}
\newtheorem{definition}{Definition}
\newtheorem{theorem}{Theorem}

\newtheorem{proposition}{Proposition}

\usepackage{xcolor}
\definecolor{darkblue}{rgb}{0.0, 0.0, 0.55}
\hypersetup{
    colorlinks=true,
    citecolor=darkblue,
    urlcolor=darkblue,
    linkcolor=darkblue
}

\begin{document}

%

%
\runningauthor{Chen, Chlenski, Munyuza, Moretti, Naesseth, Pe'er}

\twocolumn[

\aistatstitle{Variational Combinatorial Sequential Monte Carlo for Bayesian Phylogenetics in Hyperbolic Space}

\aistatsauthor{ Alex Chen$^\ddagger$ \And Phillipe Chlenski$^\ddagger$ \And  Kenneth Munyuza$^\ddagger$ \AND  Antonio Khalil Moretti$^{\dagger}$ \And Christian A. Naesseth$^{\circ}$ \And  Itsik Pe'er$^\ddagger$ }

\aistatsaddress{ Spelman College$^\dagger$ \And University of Amsterdam$^{\circ}$ \And Columbia University$^\ddagger$   } ]

\begin{abstract}
Hyperbolic space naturally encodes hierarchical structures such as phylogenies (binary trees), where inward-bending geodesics reflect paths through least common ancestors, and the exponential growth of neighborhoods mirrors the super-exponential scaling of topologies. This scaling challenge limits the efficiency of Euclidean-based approximate inference methods. Motivated by the geometric connections between trees and hyperbolic space, we develop novel hyperbolic extensions of two sequential search algorithms: Combinatorial and Nested Combinatorial Sequential Monte Carlo (\textsc{Csmc} and \textsc{Ncsmc}). Our approach introduces consistent and unbiased estimators, along with variational inference methods (\textsc{H-Vcsmc} and \textsc{H-Vncsmc}), which outperform their Euclidean counterparts. Empirical results demonstrate improved speed, scalability and performance in high-dimensional phylogenetic inference tasks.
\end{abstract}

\section{Introduction}
Phylogenetic inference, a key component of molecular biology, involves reconstructing ancestral relationships between organisms by modeling their evolutionary histories as trees. However, this task becomes computationally challenging due to the super-exponential scaling of tree topologies, which increases as \( (2N - 3)!! \) with the number of organisms. This scaling challenge often limits the efficiency of traditional approximate inference methods which rely on Euclidean geometry.

Hyperbolic space naturally encodes hierarchical structures  where inward-bending geodesics reflect paths through least common ancestors in trees, and the exponential growth of neighborhoods mirrors the super-exponential scaling of tree topologies. For example, \cite{sarkar} demonstrates that trees can be embedded into the Poincaré disk, a two-dimensional model of hyperbolic space, with arbitrarily low distortion. Motivated by geometric connections between trees and hyperbolic space, we develop novel hyperbolic extensions of two sequential search algorithms, Combinatorial and Nested Combinatorial Sequential Monte Carlo (\textsc{Csmc} and \textsc{Ncsmc})~\citep{csmc,pmlr-v161-moretti21a}. Our approach introduces consistent and unbiased estimators used to form approximate posteriors and variational inference methods (\textsc{H-Vcsmc} and \textsc{H-Vncsmc}), outperforming their Euclidean counterparts. By leveraging the structure of hyperbolic geometry, both methods significantly enhance speed, scalability and performance in high-dimensional phylogenetic Inference tasks.
Our contributions are as follows:
\begin{enumerate}
    \item We develop novel hyperbolic extensions of the \textsc{Csmc} and \textsc{Ncsmc} algorithms (\textsc{H-Csmc} and \textsc{H-Ncsmc}). We define proposal distributions for phylogenetic inference in hyperbolic space and prove that the resulting estimators are consistent and unbiased.
    \item We utilize these estimators to construct two novel approximate posteriors and corresponding variational inference methods, referred to as \textsc{H-Vcsmc} and \textsc{H-Vncsmc}.
    \item We empirically validate our methods on high-dimensional phylogenetic inference tasks, showing that they outperform existing approximate Bayesian Inference techniques, in some cases achieving speedups of 20 to 50 times by utilizing \textsc{Gpu} acceleration.
\end{enumerate}

\subsection{Related Work} 
There is a substantial body of work on Bayesian phylogenetic inference, including local search methods such as MCMC (e.g., MrBayes \citep{10.1093/bioiNformatics/17.8.754}), \textsc{ppHmc}~\citep{pmlr-v70-dinh17a}, and particle MCMC \citep{wang2020particle}. Sequential search methods include \textsc{Csmc}~\citep{csmc} and \textsc{Ncsmc}~\citep{pmlr-v161-moretti21a}. A number of Variational Inference (VI) techniques have also been developed. \textsc{Vcsmc}~\citep{pmlr-v161-moretti21a} is based on sequential search, while the majority of approaches are based on standard VI \citep{zhang2018variational,zhang2020improved,NEURIPS2022_5e956fef,NEURIPS2023_732c5757,zhang2023learnable}. 

Recent work has explored phylogenetic inference in hyperbolic space through neighbor joining. \cite{macaulay2024differentiable} introduced a differentiable variant of neighbor joining in hyperbolic space, while GeoPhy~\citep{NEURIPS2023_732c5757} combines hyperbolic leaf node embeddings with neighbor-joining for tree construction. Our method takes a probabilistic approach, using sequential node resampling based on conditional likelihoods to explore multiple tree structures simultaneously through variational sequential search in hyperbolic space.
Our work builds on two recent advancements in machine learning within hyperbolic space. \cite{NIPS2017_59dfa2df} proposed a method for embedding hierarchical structures in the Poincaré disk. Additionally, \cite{pmlr-v97-nagano19a} developed the \textit{Wrapped Normal}, a distribution on hyperbolic space with an analytically tractable density and differentiable parameters.

\vspace{-1em}
\section{Background}
We review key concepts from hyperbolic geometry and approximate Bayesian phylogenetic inference.
\vspace{-1em}

\subsection{Geometry in Riemannian Manifolds}

\textit{Hyperbolic geometry} is a type of non-Euclidean geometry defined by constant negative curvature. This geometry can be represented through four equivalent models: the hyperboloid (or Lorentz/Minkowski model), the Poincaré disk model, the Poincaré half-plane model, and the Klein model. While our approach applies to each of these representations due to their equivalence, we focus on the Poincaré disk model due to its intuitive geometric structure and simplicity for visualization.

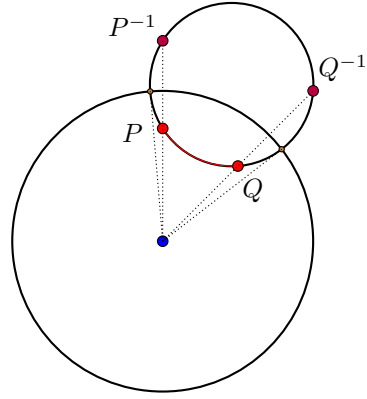
\begin{figure}
\begin{tikzpicture}[scale=2]

    \draw[thick, black] (0,0) circle [radius=1];
    \draw[thick, black] (0.4583333333333333, 1.0416666666666665) circle [radius=0.5432668671002207];
    
    \coordinate (P) at (0,0.75);
    \coordinate (Q) at (0.5,0.5);
    \coordinate (C) at (0,0);
    \coordinate (Pinv) at (0.0, 1.3333333333333333);
    \coordinate (Qinv) at (1.0, 1.0);
    \coordinate (CGeo) at (0.46,1.04);
    \coordinate (Pinter) at (-0.08305646147470835, 0.9965448430488716);
    \coordinate (Qinter) at (0.7908312604023223, 0.612034245422978);    
    \def\radius{0.54}
    
    \draw[red] (P) arc[start angle=-147.53, end angle=-85.60, radius=\radius];
    
    \draw[fill=red, draw=black] (P) circle (1pt);
    \draw[fill=red, draw=black] (Q) circle (1pt);
    \draw[fill=blue, draw=black] (C) circle (1pt);
    \draw[fill=purple, draw=black] (Pinv) circle (1pt);
    \draw[fill=purple, draw=black] (Qinv) circle (1pt);
    \draw[fill=brown, draw=black] (Pinter) circle (0.5pt);
    \draw[fill=brown, draw=black] (Qinter) circle (0.5pt);
    \draw[dash pattern=on 0.45pt off 1pt] (C)--(Pinv);
    \draw[dash pattern=on 0.45pt off 1pt] (C)--(Pinter);
    \draw[dash pattern=on 0.45pt off 1pt] (C)--(Qinv);
    \draw[dash pattern=on 0.45pt off 1pt] (C)--(Qinter);

    \node[above] at (-0.2, 1.3) {$P^{-1}$};
    \node[above] at (1.2, 1.) {$Q^{-1}$};
    \node[above] at (-0.2, .6) {$P$};
    \node[above] at (.6, .2) {$Q$};
    
\end{tikzpicture}
\centering
\label{fig:poincare_geodesics1}
\caption{Geodesic in the Poincaré disk connecting points \( P \) and \( Q \) depicted in red, with their corresponding inverses \( P^{-1} \) and \( Q^{-1} \). A circle orthogonal to the unit circle that passes through points \( P \) and \( Q \) must also intersect inverses \( P^{-1} \) and \( Q^{-1} \), which are the reflections of \( P \) and \( Q \) with respect to the unit circle.}
\end{figure}

\subsubsection{Poincaré Disk Model}
The \textit{Poincaré disk model} represents hyperbolic geometry within the unit disk in Euclidean space. This model is defined on the open unit disk, denoted as 
\begin{equation}
    \mathcal{P} \coloneqq \left\{ \bz \in \mathbb{R}^2 : \|\bz\| < 1 \right\} 
    \,.
\end{equation}
Let \( \|\bz\| \) denote the Euclidean norm of the point \( \bz \). The hyperbolic distance between two points \( \bz_1 \) and \( \bz_2 \) inside the unit disk is given by:
\begin{equation}
    d(\bz_1, \bz_2) = \text{arcosh}\left(1 + \frac{2 \|\bz_1 - \bz_2\|^2}{(1 - \|\bz_1\|^2)(1 - \|\bz_2\|^2)}\right)\,.
    \label{poincare_distances_1}
\end{equation}
\paragraph{Metric Tensor.} The \textit{metric tensor} in the Poincaré disk is defined using a conformal modification of the Euclidean metric:
\begin{equation}
    ds^2 = \frac{4(dx^2 + dy^2)}{(1-x^2-y^2)^2}\,,
\end{equation}
where $dx^2 + dy^2$ is the standard Euclidean metric in $\mathbb{R}^2$ and $4/(1-x^2-y^2)^2$ is a conformal scaling factor dependent upon distance from the origin. As points approach the boundary of the disk ($x^2 + y^2 \rightarrow 1$), distances tend to infinity and the metric becomes increasingly large. 

\paragraph{Geodesics.} In the Poincaré disk model, \textit{geodesics} which represent the ``shortest paths" in hyperbolic space are defined by arcs of circles that are orthogonal to the boundary of the disk or by diameters of the disk. Given two points \( \bz_1 \) and \( \bz_2 \) inside the disk, their geodesic can be constructed as the unique circle passing through both points and perpendicular to the disk boundary. If the points lie on a diameter, the geodesic is simply the Euclidean straight line segment connecting them. Further details are provided in \ref{geodesics}.

\subsubsection{Parallel Transport and Exponential Map}
Hyperbolic space is not a vector space, which renders the linear operations typically employed in vector spaces inapplicable. Riemannian geometry requires special tools to understand how vectors change as they move across curved surfaces. 
\textit{Parallel transport} adjusts vectors along a curve while preserving their inner product structure, ensuring consistent comparisons between different points. The \textit{exponential map} projects directions from the tangent space to the manifold along geodesics, allowing us to translate local, linear information into motion across the curved space.

\paragraph{Parallel transport.} Given two vectors $\bx, \by \in \mathcal{P}$, the parallel transport from $\bx$ to $\by$ is defined as a map $\text{PT}_{\bx \to \by}$ from the tangent plane at $\bx$, $T_{\bx}\mathcal{P}$, to the tangent plane at $\by$, $T_{\by}\mathcal{P}$. This map transports a vector in $T_{\bx}\mathcal{P}$ along the geodesic from $\bx$ to $\by$ in a parallel manner, preserving the metric tensor throughout the process. When $\bx$ is the origin $\mu_0$, parallel transport can be expressed in simple closed form: 
\begin{equation} 
\text{PT}_{\mu_0 \rightarrow \by}(\bv) \coloneqq (1 - |\by|_2^2) \bv\,, 
\end{equation} 
where $\bv$ denotes the vector being transported from the tangent space at the origin $\mu_0$ to the tangent space at point $\by$ in the Poincaré disk. Further details are provided in \cite{NEURIPS2018_dbab2adc} and Appendix \ref{parallel_transport_and_expmap}.

\paragraph{Möbius Addition.}
Möbius addition defines an operation for combining points taking into account the negative curvature of the manifold. This can be used to define a simple closed form for the exponential map.  
Within the Poincaré disk, Möbius addition reduces to:
\begin{equation}
    \bx \oplus \by \coloneqq \frac{(1 + 2 \langle \bx, \by \rangle + ||\by||_2^2)\bx + (1 - ||\bx||_2^2)\by}{1 + 2\langle \bx, \by \rangle  + ||\bx||^2_2||\by||^2_2 }\,.
\end{equation}
For additional details see \cite{ungar2009gyrovector}.

\paragraph{Exponential and Logarithmic Map.} 
The exponential map converts directions from the tangent plane into geodesic movement on the curved surface. The exponential map and its inverse, the logarithmic map, are defined on the Poincaré disk as:
\begin{align}
\exp_{\bx}^{-1}(\bv) &\coloneqq \bx \oplus \left( \tanh \left( (1 - ||\bx||_2^2)||\bv||_2 \right) \frac{\bv}{||\bv||_2} \right)\,, \\
    \log_{\bx}(\by) &\coloneqq \tanh^{-1} \left(\|\by\|_2\right)\frac{\by}{\|\by\|_2}\,.
\end{align}
See Appendix \ref{parallel_transport_and_expmap} for additional details.

\subsection{Phylogenetics}
We aim to infer a latent bifurcating tree that represents the evolutionary relationships among a set of observed molecular sequences. A phylogeny is characterized by both a tree topology, denoted by $\tau$, and a corresponding set of branch lengths $\mathcal{B}$. 
Each edge $e \in E$ has an associated \textit{branch length}, represented as $\beta(e) \in \mathbb{R}_{>0}$, with $\mathcal{B} = \{\beta(e)\}_{e \in E}$. The branch length reflects the rate of evolutionary change between two 
vertices. An \textit{ultrametric tree} assumes constant rate of evolution along all paths from root to leaf nodes. \textit{Nonclock trees} do not assume constant evolutionary rate.

\begin{figure*}[h!]
   \centering
   \begin{tikzpicture}[sloped, scale=0.8]
   \tikzset{
   dep/.style={circle,minimum size=1.5,fill=orange!20,draw=orange,line width=1.2pt,
               general shadow={fill=gray!60,shadow xshift=1pt,shadow yshift=-1pt}},
   dep/.default=1cm,
   cli/.style={circle,minimum size=1.5,fill=white,draw,line width=1.2pt,
               general shadow={fill=gray!60,shadow xshift=1pt,shadow yshift=-1pt}},
   cli/.default=1cm,
   obs/.style={circle,minimum size=1.5,fill=gray!20,draw,line width=1.2pt,
               general shadow={fill=gray!60,shadow xshift=1pt,shadow yshift=-1pt}},
   obs/.default=1cm}

       \pgfmathsetmacro{\rowAoffsetY}{4.25};
       \pgfmathsetmacro{\rowBoffsetY}{1.75};
       \pgfmathsetmacro{\rowCoffsetY}{-.75};
       
       \pgfmathsetmacro{\colAoffsetX}{-6.5};
       \pgfmathsetmacro{\colBoffsetX}{-1.5};
       \pgfmathsetmacro{\colCoffsetX}{3.25};
       \pgfmathsetmacro{\colDoffsetX}{8.25};
       
       \node(Resample) at (-4.5,6) {\textsc{Resample}};
       \node (A) at (-.95+\colAoffsetX, 0+\rowAoffsetY) {\tiny $\bullet$};
       \node (B) at (-.65+\colAoffsetX, -0.65+\rowAoffsetY) {\tiny $\bullet$};
       \draw
       (\colAoffsetX,\rowAoffsetY) circle [radius=1];
       \draw
       (A) arc[start angle=82.99, end angle=-33.44, radius=0.421];
       \node (C) at (.7+\colAoffsetX, .6+\rowAoffsetY) {\tiny $\bullet$};
       \node (D) at (-.6+\colAoffsetX, .75+\rowAoffsetY) {\tiny $\bullet$};
       
       \coordinate (aone) at (\colAoffsetX+.5,\rowAoffsetY); 
       
       \node (E) at (-.95+\colAoffsetX, 0+\rowBoffsetY) {\tiny $\bullet$};
       \node (F) at (-.65+\colAoffsetX, -0.65+\rowBoffsetY) {\tiny $\bullet$};
       \node (G) at (.7+\colAoffsetX, .6+\rowBoffsetY) {\tiny $\bullet$};
       \node (H) at (-.6+\colAoffsetX, .75+\rowBoffsetY) {\tiny $\bullet$};
       \draw[line width=1.2pt] (\colAoffsetX,\rowBoffsetY) circle [radius=1];
       \draw[line width=1.2pt] (G) arc[start angle=-54.2, end angle=-138.95, radius=0.97];
       
       \coordinate (bone) at (\colAoffsetX + 1.25, \rowBoffsetY);
       
       \node (I) at (-.95+\colAoffsetX, 0+\rowCoffsetY) {\tiny $\bullet$};
       \node (J) at (-.65+\colAoffsetX, -.65+\rowCoffsetY) {\tiny $\bullet$};
       \node (K) at (.7+\colAoffsetX, 0.6+\rowCoffsetY) {\tiny $\bullet$};
       \node (L) at (-.6+\colAoffsetX, .75+\rowCoffsetY) {\tiny $\bullet$};
       \draw[line width=0.9pt]
       (\colAoffsetX,\rowCoffsetY) circle [radius=1];
       \draw[line width=0.9pt] 
       (L) arc[start angle=33.87, end angle=-83.90, radius=0.483];
       
       \coordinate (cone) at (\colAoffsetX+1.25, \rowCoffsetY);
   
       \node(Propose) at (.75,6) {\textsc{Propose}};
       
       \node (AA) at (-.95+\colBoffsetX, 0+\rowAoffsetY) {\tiny $\bullet$};
       \node (BB) at (-.65+\colBoffsetX, -.65+\rowAoffsetY) {\tiny $\bullet$};
       \node (CC) at (.7+\colBoffsetX, .6+\rowAoffsetY) {\tiny $\bullet$};
       \node (DD) at (-.6+\colBoffsetX, .75+\rowAoffsetY) {\tiny $\bullet$};
       \draw (\colBoffsetX,\rowAoffsetY) circle [radius=1];
       \draw (CC) arc[start angle=-54.2, end angle=-138.95, radius=0.97];
       
       \coordinate (atwo) at (-3, \rowAoffsetY);
       \coordinate (athree) at (-.25, \rowAoffsetY);

       \node (EE) at (-.95+\colBoffsetX, 0+\rowBoffsetY) {\tiny $\bullet$};
       \node (FF) at (-.65+\colBoffsetX, -.65+\rowBoffsetY) {\tiny $\bullet$};
       \node (GG) at (.7+\colBoffsetX, .6+\rowBoffsetY) {\tiny $\bullet$};
       \node (HH) at (-.6+\colBoffsetX, .75+\rowBoffsetY) {\tiny $\bullet$};
       \draw (\colBoffsetX,\rowBoffsetY) circle [radius=1];
       \draw (GG) arc[start angle=-54.2, end angle=-138.95, radius=0.97];
       
       \coordinate (btwo) at (-3, \rowBoffsetY);
       \coordinate (bthree) at (-.25, \rowBoffsetY);
       \draw[->,dashed] (bone) -- (btwo);
       \draw[->,dashed] (bone) -- (atwo);
       
       \node (II) at (-.95+\colBoffsetX, 0+\rowCoffsetY) {\tiny $\bullet$};
       \node (JJ) at (-.65+\colBoffsetX, -.65+\rowCoffsetY) {\tiny $\bullet$};
       \node (KK) at (.7+\colBoffsetX, .6+\rowCoffsetY) {\tiny $\bullet$};
       \node (LL) at (-.6+\colBoffsetX, .75+\rowCoffsetY) {\tiny $\bullet$};
       \node (kkll) at (\colBoffsetX, \rowCoffsetY) {};
       \draw (\colBoffsetX,-.75) circle [radius=1];
       \draw (LL) arc[start angle=33.87, end angle=-83.90, radius=0.483];
       
       \coordinate (ctwo) at (-3., \rowCoffsetY);
       \coordinate (cthree) at (-.25, \rowCoffsetY);
       \draw[->,dashed] (cone) -- (ctwo);
       
       \node (U) at (-.95+\colCoffsetX, 0+\rowAoffsetY) {\tiny $\bullet$};
       \node (V) at (-.65+\colCoffsetX, -.65+\rowAoffsetY) {\tiny $\bullet$};
       \node (W) at (.7+\colCoffsetX, .6+\rowAoffsetY) {\tiny $\bullet$};
       \node (X) at (-.6+\colCoffsetX, .75+\rowAoffsetY) {\tiny $\bullet$};
       \node (Vmid) at (.075+\colCoffsetX, .42+\rowAoffsetY) {};
       \coordinate (afour) at (2,4.25);
       \draw[->, dashed] (athree) -- (afour);
       \draw (3.25,4.25) circle [radius=1];
       \draw (W) arc[start angle=-54.2, end angle=-138.95, radius=0.97];
       \draw (Vmid) arc[start angle=-42.5, end angle=-88.8, radius=1.33];
       
       \node (M) at (-.95+\colCoffsetX, 0+\rowBoffsetY) {\tiny $\bullet$};
       \node (N) at (-.65+\colCoffsetX, -.65+\rowBoffsetY) {\tiny $\bullet$};
       \node (O) at (.7+\colCoffsetX, .6+\rowBoffsetY) {\tiny $\bullet$};
       \node (P) at (-.6+\colCoffsetX, .75+\rowBoffsetY) {\tiny $\bullet$};
       \coordinate (bfour) at (2,1.75);
       \draw[->, dashed] (bthree) -- (bfour);
       \draw (3.25,1.75) circle [radius=1];
       \draw (O) arc[start angle=-54.2, end angle=-138.95, radius=0.97];
       \draw (M) arc[start angle=82.99, end angle=-33.44, radius=0.421];
       
       \node (Q) at (-.95+\colCoffsetX, 0+\rowCoffsetY) {\tiny $\bullet$};
       \node (R) at (-.65+\colCoffsetX, -.65+\rowCoffsetY) {\tiny $\bullet$};
       \node (S) at (.7+\colCoffsetX, .6+\rowCoffsetY) {\tiny $\bullet$};
       \node (T) at (-.6+\colCoffsetX, .75+\rowCoffsetY) {\tiny $\bullet$};
       \draw (T) arc[start angle=33.87, end angle=-83.90, radius=0.483];
       \draw (3.25,-.75) circle [radius=1];
       \node (midTQ) at (-.575+\colCoffsetX, .25+\rowCoffsetY) {};
       \draw (midTQ) arc[start angle=29.44, end angle=-38.97, radius=0.803];
       
       \coordinate (cfour) at (2,-.75);
       \draw[->, dashed] (cthree) -- (cfour);
       
       \node(Weighting) at (5.75,6) {\textsc{Weighting}};
       \node (Uu) at (-.95+\colDoffsetX, 0+\rowAoffsetY) {\tiny $\bullet$};
       \node (Vv) at (-.65+\colDoffsetX, -.65+\rowAoffsetY) {\tiny $\bullet$};
       \node (Ww) at (.7+\colDoffsetX, .6+\rowAoffsetY) {\tiny $\bullet$};
       \node (Xx) at (-.6+\colDoffsetX, .75+\rowAoffsetY) {\tiny $\bullet$};
       \node (Vvmid) at (.075+\colDoffsetX, .42+\rowAoffsetY) {};
       \draw[line width=0.9pt] (Ww) arc[start angle=-54.2, end angle=-138.95, radius=0.97];
       \draw[line width=0.9pt] (Vvmid) arc[start angle=-41.5, end angle=-88.8, radius=1.337];
       \draw[line width=0.9pt] (8.25,4.25) circle [radius=1];
       
       \coordinate (afive) at (4.5,4.25);
       \coordinate (asix) at (7,4.25);
       \draw[->, dashed] (afive) -- (asix);
       
       \node (Mm) at (-.95+\colDoffsetX, 0+\rowBoffsetY) {\tiny $\bullet$};
       \node (Nn) at (-.65+\colDoffsetX, -.65+\rowBoffsetY) {\tiny $\bullet$};
       \node (Oo) at (.7+\colDoffsetX, .6+\rowBoffsetY) {\tiny $\bullet$};
       \node (Pp) at (-.6+\colDoffsetX, .75+\rowBoffsetY) {\tiny $\bullet$};
       \draw (Oo) arc[start angle=-54.2, end angle=-138.95, radius=0.97];
       \draw (Mm) arc[start angle=82.99, end angle=-33.44, radius=0.421];
       \draw (8.25,1.75) circle [radius=1];
       
       \coordinate (bfive) at (4.5,1.75);
       \coordinate (bsix) at (7,1.75);
       \draw[->, dashed] (bfive) -- (bsix);
       
       \node (Qq) at (-.95+\colDoffsetX, 0+\rowCoffsetY) {\tiny $\bullet$};
       \node (Rr) at (-.65+\colDoffsetX, -.65+\rowCoffsetY) {\tiny $\bullet$};
       \node (Ss) at (.7+\colDoffsetX, .6+\rowCoffsetY) {\tiny $\bullet$};
       \node (Tt) at (-.6+\colDoffsetX, .75+\rowCoffsetY) {\tiny $\bullet$};
       \draw[line width=1.2pt] 
       (Tt) arc[start angle=33.87, end angle=-83.90, radius=0.483];
       \node (midQqTt) at (-.575+\colDoffsetX, .25+\rowCoffsetY) {};
       \draw[line width=1.2pt] 
       (midQqTt) arc[start angle=29.44, end angle=-38.97, radius=0.803];
       \draw[line width=1.2pt] 
       (8.25,-.75) circle [radius=1];
       
       \coordinate (cfive) at (4.5,-.75);
       \coordinate (csix) at (7,-.75);
       \draw[->, dashed] (cfive) -- (csix);
       
   \end{tikzpicture}
    \caption{Overview of \textsc{Csmc} in the Poincaré disk. Each iteration involves three steps: (1) \textsc{Resample} partial states by importance weights, (2) \textsc{Propose} extensions by connecting two trees, and (3) \textsc{Weight} using Felsenstein's pruning. The progression of $K = 3$ samples are shown with taxa embedded near the boundary of the disk. Linewidths represent weights.}

    \label{fig:trellis}
\end{figure*}
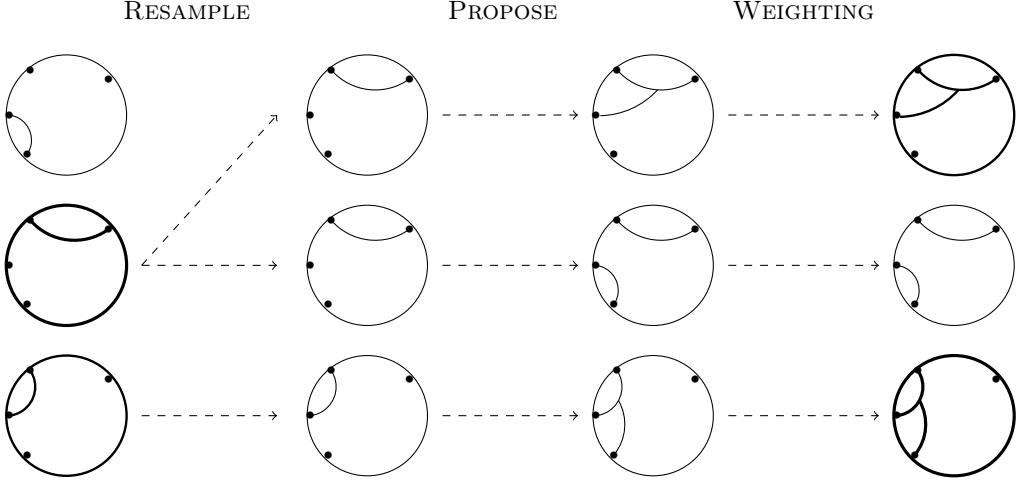

\paragraph{Bayesian Phylogenetic Inference.}
Let $\mathbf{Y} = \{Y_1,\cdots,Y_S\} \in \Omega^{N \times S}$ represent the observed molecular sequences, where $\Omega$ denotes the set of characters, $S$ is the sequence length, and $N$ is the number of species. Bayesian phylogenetic inference involves specifying a prior distribution and a likelihood function over the tree topology $\tau$, the set of branch lengths $\mathcal{B}$, and the model parameters $\theta$ to express the joint posterior:
\begin{equation}
    P_{\theta}(\mathcal{B},\tau|\mathbf{Y}) = \frac{P_{\theta}(\mathbf{Y}|\tau,\mathcal{B}) P_{\theta}(\tau,\mathcal{B})}{P_{\theta}(\mathbf{Y})}.
    \label{eq:posterior}
\end{equation}
Here, the prior is uniform over the possible topologies, while the branch lengths follow independent exponential distributions with rate $\lambda_{bl}$. Each site in the sequence evolves independently according to a continuous-time Markov chain characterized by a rate matrix $\mathbf{Q}$.
Let $\zeta_{v,s}$ represent the genomic state at site $s$ for species $v$, and define the evolutionary model for branch $b(v \to v')$ as follows:
\begin{equation}
P_{\theta}(\zeta_{v',s} = j|\zeta_{v,s} = i) = \exp\left(\beta(e)\mathbf{Q}_{i,j}\right).
\label{genomic_state}
\end{equation}
The likelihood for a given phylogeny is then expressed as $P_{\theta}(\mathbf{Y}|\tau,\mathcal{B}) = \prod_{i=1}^{S} P_{\theta}(Y_i|\tau,\mathcal{B})$, and can be computed efficiently in linear time using the sum-product algorithm or Felsenstein’s pruning algorithm~\citep{Felsenstein:1981:J-Mol-Evol:7288891}:
\begin{align}
    P_{\theta}(&\mathbf{Y}|\tau,\mathcal{B}) 
    \coloneqq \prod_{i=1}^{S} \sum_{a^i} \eta(a^i_{\rho}) \prod_{(u,v) \in E(\tau)} \exp\left(-\beta_{u,v} \mathbf{Q}_{a_u^i,a_v^i} \right),
\end{align}
where $\rho$ is the root node, $a^i_u$ is the character assigned to node $u$, $E(\tau)$ denotes the set of edges in the tree $\tau$, and $\eta$ is the stationary distribution of the Markov chain. The normalization constant $P_{\theta}(\mathbf{Y})$ involves marginalizing over $(2N-3)!!$  topologies, which is computationally infeasible~\citep{semple2003phylogenetics}.

\subsection{Variational Inference}
Variational Inference (\textsc{VI}) is used to approximate the posterior distribution $P_\theta(\mathcal{B}, \tau | \mathbf{Y})$ when directly marginalizing latent variables is not feasible. By introducing a simpler, tractable distribution $Q_\phi(\mathcal{B}, \tau | \mathbf{Y})$, we can derive a lower bound for the log-likelihood:
\begin{equation}
\log P_\theta(\mathbf{Y}) \geq \mathcal{L}_{\text{ELBO}}(\theta, \phi, \mathbf{Y}) \coloneqq \underset{Q}{\mathbb{E}}\left[\log \frac{P_\theta(\mathbf{Y}, \mathcal{B}, \tau)} {Q_\phi(\mathcal{B}, \tau | \mathbf{Y})}\right]\,. \label{ELBO}
\end{equation}
Auto-Encoding Variational Bayes (\textsc{AEVB})~\citep{kingma2013autoencoding} trains both the distributions $Q_\phi(\mathcal{B}, \tau | \mathbf{Y})$ and $P_\theta(\mathbf{Y}, \mathcal{B}, \tau)$ simultaneously. The expectation in Eq. \ref{ELBO} is estimated by averaging Monte Carlo samples from $Q_\phi(\mathcal{B}, \tau | \mathbf{Y})$, with the samples reparameterized using a deterministic function of a random variable independent of $\phi$. 

Constructing a tractable approximation $Q_\phi(\mathcal{B}, \tau | \mathbf{Y})$ can be difficult, prompting the use of \textsc{Csmc}.

\subsection{Combinatorial Sequential Monte Carlo}

\textsc{Csmc} approximates a sequence of probability spaces that converge to the target distribution over \(N-1\) steps, indexed by $r$. It uses sequential importance resampling to estimate the unnormalized target distribution $\pi$ and its normalization constant $\|\pi\|$, yielding the desired posterior. This is achieved by sampling $K$ \textit{partial states} (see \ref{partial_states}) $\{s_{r}^k\}_{k=1}^{K} \in \mathcal{S}_r$ at each step $r$, gradually constructing an approximation to the full distribution:
\begin{equation}
    \widehat{\pi}_{r} = \|\widehat{\pi}_{r-1}\|\frac{1}{K}\sum\limits_{k=1}^{K}w_{r}^k\delta_{s_r^k}(s) \qquad \forall s \in \mathcal{S}\,.
\end{equation}
Unlike typical \textsc{Smc} methods, \textsc{Csmc} works with both a combinatorial set of tree topologies and the continuous branch lengths. At each step, Monte Carlo samples are resampled, keeping them concentrated in regions of higher probability with importance weights given by:
\begin{equation}
 w_{r}^k = w(s_{r-1}^{a_{r-1}^k},s_{r}^k) = \frac{\pi(s_{r}^k)}{\pi(s_{r-1}^{a_{r-1}^k})}\cdot \frac{\nu^{-}(s_{r-1}^{a_{r-1}^k})}{q(s_{r}^k |s_{r-1}^{a_{r-1}^k})}\,,
 \label{eq:weights}
\end{equation}
where $q(s_{r}^k |s_{r-1}^{a_{r-1}^k})$ is a proposal distribution, and $a_{r-1}^k \in \{1, \cdots, K\}$ denotes the resampled ancestor indices chosen with probabilities proportional to the weights: $\mathbb{P}(a_{r-1}^k = i) = \nicefrac{w_{r-1}^i}{\sum_{l=1}^K w_{r-1}^l}$. The factor $\nu^{-}$ corrects for overcounting. After resampling, new states are proposed and extended (see Fig.~\ref{fig:trellis}). This gives an unbiased estimate of $\|\pi\|$ converging in $L_2$ norm:
\begin{equation}
    \widehat{\mathcal{Z}}_{CSMC} \coloneqq \|\widehat{\pi}_{R}\| = \prod\limits_{r=1}^{R}\left(\frac{1}{K} \sum\limits_{k=1}^{K}w_{r}^k\right) \rightarrow \|\pi \|.
    \label{eq:smcmarginallikelihood}
\end{equation}
\subsection{Variational Combinatorial Sequential Monte Carlo}
\textsc{Vcsmc}~\citep{pmlr-v161-moretti21a} learns model and proposal parameters by maximizing a lower bound on $\|\pi\|$, using \textsc{Csmc} (or \textsc{Ncsmc}) as an unbiased estimator of $\|\pi\|$:
\begin{align}
     \mathcal{L}_{VCSMC} 
     \coloneqq \underset{Q}{\mathbb{E}}\left[\log \widehat{ \mathcal{Z}}_{CSMC} \right]\,.
     \label{elbo_vcsmc}
\end{align}
In \textsc{Csmc}, nodes are sampled to coalesce uniformly, even when many resulting topologies have low probabilities. \textsc{Ncsmc}~\citep{pmlr-v161-moretti21a} uses information from one future iteration to guide the exploration of partial states, providing an exact approximation of the \textit{locally optimal proposal}~\citep{doucet2000sequential,naesseth2019elements}. See \ref{hyperbolicNCSMC} for details of \textsc{Ncsmc} and \textsc{V-Ncsmc}.

\section{Methodology}
In \ref{embedding_points}, we outline the methodology for embedding sequences into hyperbolic space. In \ref{hyp_csmc} we introduce the algorithm that samples embedded points to construct a tree, while \ref{approx_post} defines the approximate posterior. In \ref{scoring_tree_likelihoods}, we describe the process of scoring a tree and assigning importance weights. Finally, \ref{secton_hncsmc} extends the process to \textsc{Ncsmc}.

\subsection{Embedding Sequences}
\label{embedding_points}
We embed sequences into the Poincaré disk by mapping sequences \( Y_n \) into hyperbolic space while preserving pairwise relationships \citep{NIPS2017_59dfa2df}. This involves minimizing a loss function to align hyperbolic distances in the Poincaré disk with original Hamming distances:
\begin{equation}
\mathcal{L} = \sum_{i,j} \left( d_{\mathbb{H}}(\bz_i, \bz_j) - d_H(\bx_i, \bx_j) \right)^2,
\end{equation}
where \( d_{\mathbb{H}}(\bz_i, \bz_j) \) is the hyperbolic distance between embedded points \( \bz_i \) and \(\bz_j \) in the Poincaré disk, and \( d_H(\bx_i, \bx_j) \) is the Hamming distance between original points \( \bx_i \) and \( \bx_j \). Alternatively, a multilayer perceptron (MLP) can be employed to map sequences directly into the Poincaré disk~\citep{NEURIPS2017_59dfa2df}.

\subsection{Hyperbolic CSMC}
\label{hyp_csmc}
After embedding the leaf nodes in the Poincaré disk, two nodes are sampled for coalescence. To assign the parent node's embedding, we parameterize the geodesic arc \(\gamma(\mathbf{v}_{L,r}, \mathbf{v}_{R,r})\) between the two sampled child nodes \(\mathbf{v}_{L,r}\) and \(\mathbf{v}_{R,r}\), and assign \(\mathbf{v}_p\) to the point on this curve that is closest to the origin $\mu_0$:
\begin{equation}
\bv_p \gets \argmin_{\mathbf{v} \in \gamma(\mathbf{v}_{L,r}, \mathbf{v}_{R,r})} d(\mu_0, \bv)\,.
\end{equation}
The minimum geodesic distance is obtained using non-ambient distance coordinates. Conveniently, the distance to the origin is monotonic in the Euclidean norm; therefore, with a parameterization of the geodesic, we can simply minimize the Euclidean distance $d(\mu_0, \bv_p)$.

Once the nodes to coalesce have been sampled, the lengths of the left and right branches, denoted by \(\beta_{L,r}\) and \(\beta_{R,r}\), can be defined as the geodesic distances to the embedding of \(\mathbf{v}_p\). 
However, this results in identical branch lengths for trees with equivalent topologies. Phylogenetic models typically account for variation in branch lengths across equivalent topologies. To account for this, we use the embedded parent node \(\mathbf{v}_p\) to define a proposal distribution on the Poincaré disk.

\begin{algorithm}[t]
\caption{Sampling Trees in Hyperbolic Space}
\begin{algorithmic}[1]
    \State \textbf{Input:} Coalescing points $\mathbf{v}_{L,r}$, $\mathbf{v}_{R,r}$ on $\mathcal{P}$
    
    \State Find the point on the geodesic arc between $\mathbf{v}_{L,r}$ and  $\mathbf{v}_{R,r}$ that is closest to the origin:
    \[
    \bv_p \gets \argmin_{\bv \in \gamma(\bv_{L,r}, \bv_{R,r})} 
    d(\mu_0, \bv)
    \]
    \State Sample a normal distribution in $\mathbb{R}^2$ to yield a vector in the tangent plane at the origin, $T_{\mu_0}\mathcal{P}$:
    \[
    \mathbf{v}^k \sim \mathcal{N}(0, \Sigma)
    \]    
    \State Parallel transport the tangent vector from the origin, $\mathbf{\mu}_0$, to the point $\bv_p$:
    \[
    \tilde{\bv}^k \gets PT_{\mathbf{\mu}_0 \rightarrow \bv_p}(\mathbf{v}^k) = (1 - \|\bv_p\|_2^2)\,\mathbf{v}^k
    \]
    
    \State Apply the exponential map to the transported vector:
    \begin{align}
            \dot{\bv}^k &\gets \exp_{\bv_p}(\tilde{\bv}^k) \nonumber \\&= \bv_p \oplus \left( \tanh \left( (1 - \|\bv_p\|_2^2)\, \|\tilde{\bv}^k\|_2 \right) \frac{\tilde{\bv}^k}{\|\tilde{\bv}^k\|_2} \right) \nonumber
    \end{align}

    \State Compute the left and right branch lengths using geodesic distance (Eq. \ref{poincare_distances_1}):
    \[
    \beta_{L,r}^k \gets d(\mathbf{v}_{L,r}, \dot{\mathbf{v}}^k)\,,\,
    \beta_{R,r}^k \gets d(\mathbf{v}_{R,r}, \dot{\mathbf{v}}^k)
    \]
    
    \State \textbf{Output:} Sampled parent node embedding $\dot \bv^k$ and branch lengths $\beta_{L,r}^k$ , $\beta_{R,r}^k$

\end{algorithmic}
\label{sample_parent_embeddings}
\end{algorithm}

\subsubsection{Sampling Branch Lengths}
\label{proposalDist}
The Lebesgue measure is uniform in $\mathbb{R}^2$; however, the hyperbolic measure is non-uniform in the Poincaré disk. As \(x^2 + y^2 \rightarrow 1\), the metric \(ds^2 \rightarrow \infty\). Moving toward the disk boundary causes the measure of a grid square to increase exponentially, which has significant implications for sampling. 

We define a proposal distribution over parent node embeddings and branch lengths
by adapting the \textit{Wrapped Normal} ($\mathcal{WN}$) distribution introduced by \cite{pmlr-v97-nagano19a}. Algorithm \ref{sample_parent_embeddings} provides the pseudocode for sampling parent node embeddings in hyperbolic space. Given two coalescing embedded points $\bv_{L,r}^k$ and $\bv_{R,r}^k$ where $(L,R)$ denotes the left and right child nodes and $r$ denotes the rank event, the subroutine begins by identifying the point in the geodesic between $\bv_{L,r}^k$ and $\bv_{R,r}^k$ that is closest to the origin (Step 2). To sample the location of a parent node, we sample in the tangent plane at the origin where the geometry is simpler (Step 3), and then use parallel transport to move the samples to the desired mean (Step 4). Next, we apply the exponential map to project the samples back onto the manifold (Step 5). Finally, lengths of the left and right branches $\beta_{L,r}^k,$ $\beta_{R,r}^k$ are obtained using the geodesic distance between the embedding of the sampled parent and the left and right children (Step 6). Pseudocode for the full hyperbolic \textsc{Csmc} (\textsc{H-Csmc}) procedure is provided in Algorithm~\ref{alg:hcsmc} of \ref{hyperboliccsmc}.

\begin{figure}
    \centering
    \includegraphics[width=0.45\linewidth]
    {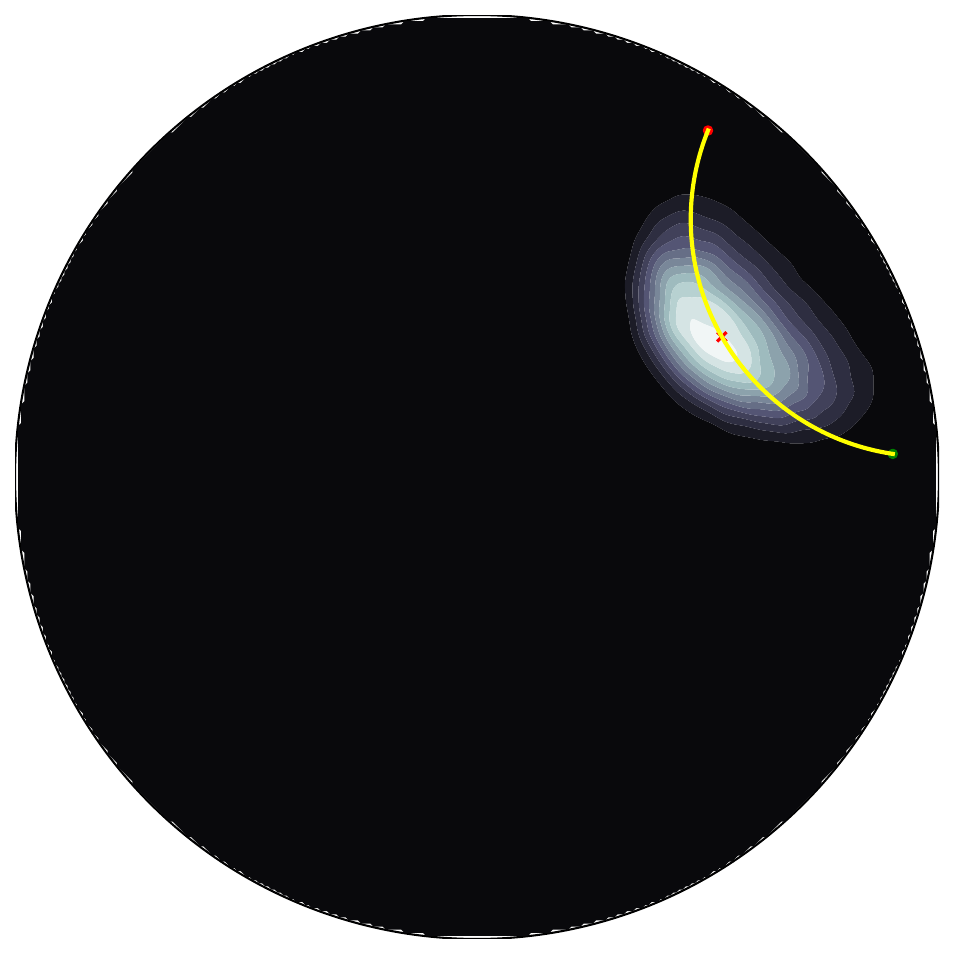}
    \quad
    \includegraphics[width=0.45\linewidth]
{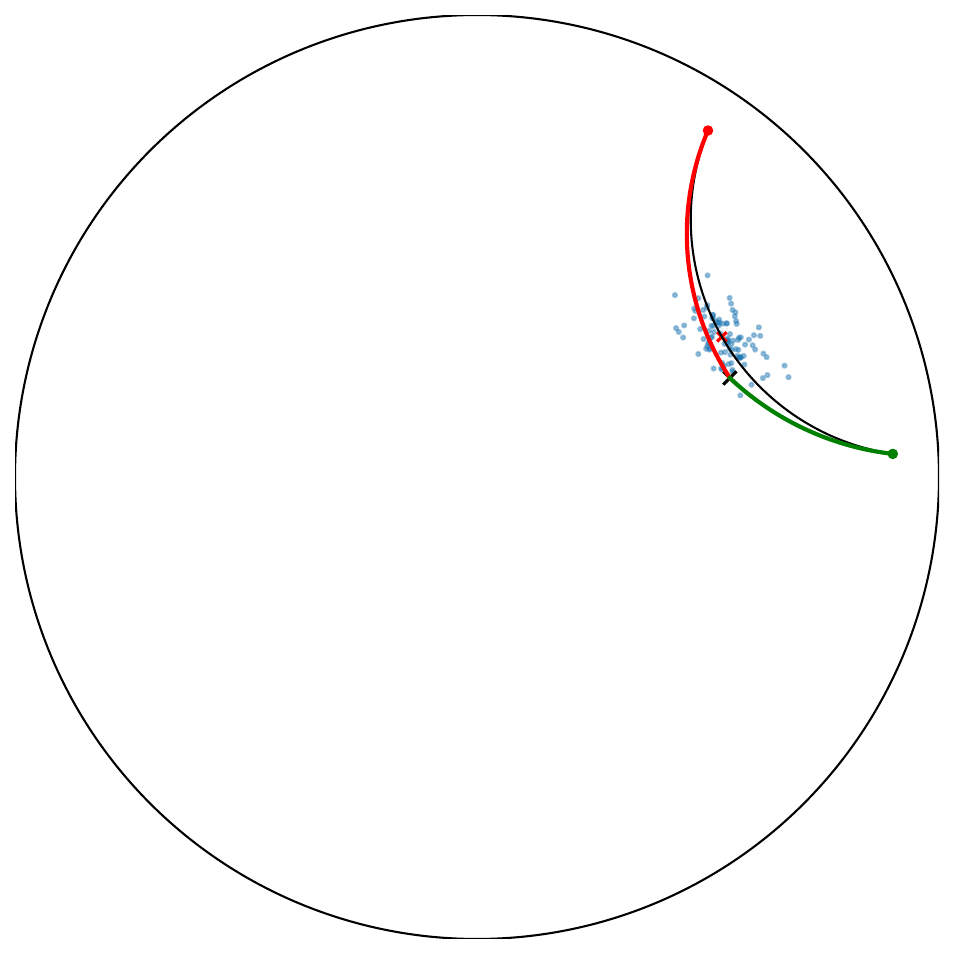}
    
    \caption{%
    \textbf{Left:} Contours of the $\mathcal{WN}$  distribution over parent node embedding for two coalescing leaves. \textbf{Right:} After drawing $\bv^k \sim \mathcal{N}(0,\Sigma)$, the vector is projected via \( \text{proj}_{\bv_p} \coloneqq \exp_{\bv_p} \circ \, \text{PT}_{\mu_0 \rightarrow _{\bv_p}} \). Left and right branch lengths are mapped via geodesic distance. 
}

    \label{fig:wrappednormal}
    \vspace{-1em}
\end{figure}

Fig.~\ref{fig:wrappednormal} (left) shows the contours of a $\mathcal{WN}$ distribution on the Poincaré disk for the parent node $\dot \bv^k$ of two child nodes, $\bv_{L,r}^k$ and $\bv_{R,r}^k$. Note the curved shape due to the concentration of the measure with respect to the geodesic distance. Fig.~\ref{fig:wrappednormal} (right) displays samples from the proposal distribution \(\dot \bv^k \sim \mathcal{WN}(\bv_p,\Sigma)\), which implicitly defines a distribution over the branch lengths $\beta_{L,r}^k$ and $\beta_{R,r}^k$.
Let $\text{proj}_{\bv_p} \coloneqq \exp_{\bv_p} \circ \, \text{PT}_{\mu_0 \rightarrow _{\bv_p}}$. The log density can be expressed as follows:
\begin{equation}
    \log p(\dot \bv^k) = \log p(\bv_p) - \log \text{det} \left( \frac{\partial \,\text{proj}_{\bv_p}}{\partial\, \bv_p}\right)\,.
\end{equation}
The determinant can be computed in $\mathcal{O}(n)$:
\begin{align}
    \text{det} &\left( \frac{\partial \,\text{proj}_{\bv_p}(\bv^k)}{\partial\, \bv_p}\right)  \\ &= 
    \text{det} \left( \frac{\partial \,\exp_{\bv_p}(\tilde{\bv}^k)}{\partial\, \tilde{\bv}^k}\right) \cdot 
    \text{det} \left( \frac{\partial \,\text{PT}_{\mu_0 \rightarrow _{\bv_p}}(\bv^k)}{\partial\, \bv_p}\right)\,.\nonumber
\end{align}
Since each operation is differentiable, we can backpropagate through all necessary transformations.

\subsection{Approximate Posterior}
\label{approx_post}
The proposal returned by \textsc{H-Vcsmc} factorizes as:
\begin{align}
     &Q_{}\left(s_{1:R}^{1:K}, a_{1:R-1}^{1:K}\right)
    \coloneqq   \\ 
    &
    \prod\limits_{k=1}^{K}q_{}(s_{1}^k)\times 
    \prod\limits_{r=2}^{R}\prod\limits_{k=1}^{K} \left[
    \frac{w_{r-1}^{a_{r-1}^k}}{\sum_{l=1}^K w_{r-1}^l} \cdot
    q_{}\left(s_{r}^k|s_{r-1}^{a_{r-1}^k}\right)
     \right]
    \, , \nonumber
    \label{full_proposal}
\end{align}

where the term for each step is the product of \textsc{Uniform} and $\mathcal{WN}$ terms for topologies and embeddings:
\begin{align}
    q_{}(s_r^k|s_{r-1}^{a_{r-1}^k})  &\coloneqq q(\tau_r^k|\tau_{r-1}^{a_{r-1}^k})\cdot q(\dot \bv_r^k|\dot \bv_{r-1}^{a_{r-1}^k},\tau_{r-1}^{a_{r-1}^k})\,.
\end{align}

The probability that a point came from the \( \mathcal{WN} \) distribution is defined by inverting $\text{proj}_{\bv_p}$ using parallel transport and the logarithmic map. 
Conveniently, since branch lengths depend only on $\dot \bv^k$, and $\dot \bv_k = \text{proj}_{\bv_p}$ we can evaluate the probability $N(\bv^k|\mu_0,\Sigma)$ under the appropriate change of coordinates. See ~\ref{branch_lengths} for the details of defining and evaluating \( q(\dot \bv_r^k|\dot \bv_{r-1}^{a_{r-1}^k},\tau_{r-1}^{a_{r-1}^k}) \).

\vspace{-.5em}
\subsection{Scoring Tree Likelihoods}
\label{scoring_tree_likelihoods}
\vspace{-.5em}

Algorithm \ref{sample_parent_embeddings} maps a point in the Poincaré disk to branch lengths, however the target $\pi$ is defined on $(\beta, \tau)$, while the proposal $q$ is defined on $(\dot \bv, \tau)$. To define importance weights, a change of variables is needed between branch lengths and the sampled points:
\begin{equation}
    w_{r}^k \coloneqq \dfrac{\pi(s_{r}^k)}{\pi(s_{r-1}^{a_{r-1}^k})}\cdot \dfrac{\nu^{-}(s_{r-1}^{a_{r-1}^k})}{q(s_{r}^k |s_{r-1}^{a_{r-1}^k})}\,\cdot \,
 \text{det} \left( \dfrac{\partial \beta_{i,r}^k(\dot \bv^k)}{\partial \dot \bv^k}\right)\,,
 \label{eq:hyperbolic_weights}
\end{equation}
where $i \in \{L, R\}$ denotes left or right child nodes and \( \beta_{i,r}^k(\dot \bv^k) = d(\dot \bv^k, \bv_{i,r}^k) \) are the sampled branch lengths, defined as the hyperbolic distance \( d(\cdot,\cdot) \) from Eq.~\ref{poincare_distances_1}.

\paragraph{Theoretical Justification.}
\textsc{H-Csmc} is an \textsc{Smc} algorithm on the extended space of all random variables generated by Algorithm~\ref{alg:hcsmc}. Therefore, it keeps the favorable properties of \textsc{Csmc}, such as unbiasedness of the normalization constant estimate and asymptotic consistency. The key property that ensures this for \textsc{H-Csmc} is \emph{proper weighting} \citep{liu2008,naesseth2015nested,naesseth2016high}. 

\begin{definition}[Proper Weighting]
We say that the random pair $(s_r, w_r)$ are \emph{properly weighted} for the unnormalized distribution $\frac{\pi(s_r) \nu^-(s_{r-1})}{\pi(s_{r-1})}$ if $w_r \geq 0$ almost surely, and for all measurable functions $h$,
\begin{align}
    \mathbb{E}[w_r h(s_r)] &= \int h(s_r) \frac{\pi(s_r) \nu^-(s_{r-1})}{\pi(s_{r-1})} \, \mathrm{d}s_r.
\end{align}
\end{definition}

\begin{proposition}[Proper Weighting]
The particles $s_r^k$ and weights $w_r^k$ 
are properly weighted for $\frac{\pi(s_r) \nu^-(s_{r-1})}{\pi(s_{r-1})}$.
\label{thm:pw}
\end{proposition}
\vspace{-2em}
\begin{proof}
\begin{align}
    &\underset{q}{\mathbb{E}}[w_r^k h(s_r^k)] = \int q(s_{r}^k |s_{r-1}^{a_{r-1}^k})\, w_r^k \, h(s_r^k)\, ds_r^k \nonumber \\
    &= \sum_{\tau} \int h( \dot \bv_r^k,\tau_r^k) \dfrac{\pi( \dot \bv_r^k,\tau_r^k)\nu^{-}(\dot \bv_{r-1}^{a_{r-1}^k},\tau_{r-1}^{a_{r-1}^k})}{\pi(\dot \bv_{r-1}^{a_{r-1}^k},\tau_{r-1}^{a_{r-1}^k})} \nonumber \\
 &\qquad\qquad\qquad\qquad\times \text{det} \left( \dfrac{\partial \beta_{i,r}^k(\dot \bv^k)}{\partial \dot \bv^k}\right) d\dot \bv^k\nonumber \\
  &= \sum_{\tau} \int h(\beta_{i,r}^k,\tau_r^k) \dfrac{\pi(\beta_{i,r}^k,\tau_r^k)\nu^{-}(\beta_{i,r-1}^{a_{r-1}^k},\tau_{r-1}^{a_{r-1}^k})}{\pi(\beta_{i,r-1}^{a_{r-1}^k},\tau_{r-1}^{a_{r-1}^k})} \, \mathrm{d}\beta_{i,r}^k \nonumber \\
 &= \int h(s_r^k) \frac{\pi(s_r^k) \nu^-(s_{r-1}^{a_{r-1}^k})}{\pi(s_{r-1}^{a_{r-1}^k})} \, \mathrm{d}s_r^k \nonumber
\end{align}
\end{proof}

\subsection{Hyperbolic Nested CSMC}
\label{secton_hncsmc}

We extend \textsc{Ncsmc} \citep{pmlr-v161-moretti21a,yang2024variational} to hyperbolic space through marginalizing over one-step lookahead partial states. 
This extension, referred to as \textsc{H-Ncsmc}, is a hyperbolic adaptation of \textsc{Ncsmc}, retaining favorable properties such as unbiasedness and consistency. Pseudocode for \textsc{H-Ncsmc} is presented in Algorithm \ref{hncsmc1} and detailed in \ref{hyperbolicNCSMC}. 

\begin{algorithm}[t!]
   \caption{Nested Combinatorial Sequential Monte Carlo with Poincaré Embeddings}
   \begin{algorithmic}[1]
       \State \textbf{Input:} $\mathbf{Y} \in \Omega^{N \times M}$, $\theta = (\mathbf{Q},\{\lambda_i\}_{i=1}^{|E|})$; embed $Y_n \rightarrow f(Y_n) \in \mathcal{P}$
       \State Initialize: $\forall k$, $s_{0}^k \leftarrow \perp$, $w_{0}^k \leftarrow 1/K$
       \For{$r = 0$ \textbf{to} $R = N-1$}
           \For{$k = 1$ \textbf{to} $K$}
               \State Resample: $\mathbb{P}(a_{r-1}^k = i) = w_{r-1}^i/\sum_{l=1}^{K} w_{r-1}^l$
               \For{$j = 1$ to $L = \binom{N-r}{2}$}
                   \For{$m = 1$ to $M$}
                       \State Find closest point: \[\qquad\qquad \mathbf{v}_p^{k,m}[j] \gets \argmin_{\mathbf{v} \in \gamma(\mathbf{v}_{L,r}, \mathbf{v}_{R,r})} d(\mu_0, \mathbf{v})\]
                       \State Sample tangent: $\mathbf{v}^{k,m}[j] \sim \mathcal{N}(0, \Sigma)$
                       \State Transport: \[\qquad\qquad\qquad \tilde{\mathbf{v}}^{k,m}[j] \gets (1 - \|\mathbf{v}_p^{k,m}[j]\|_2^2)\mathbf{v}^{k,m}[j]\]
                       \State Map to disk: $\dot{\mathbf{v}}^{k,m}[j] \gets \mathbf{v}_p^{k,m}[j] \oplus$
                       \begin{align*}
                       &\qquad
                       \tanh \big( (1 - \|\mathbf{v}_p^{k,m}[j]\|_2^2) 
                        \cdot \|\tilde{\mathbf{v}}^{k,m}[j]\|_2 \big) \frac{\tilde{\mathbf{v}}^{k,m}[j]}{\|\tilde{\mathbf{v}}^{k,m}[j]\|_2} 
                       \end{align*}
                       \State Branch lengths: \begin{align*}
                           \qquad\qquad \mathcal{B}_{L,r}^{k,m}[j] &\gets d(\mathbf{v}_{L,r}^{k,m}[j], \dot{\mathbf{v}}^{k,m}[j])\\
                           \qquad \qquad\mathcal{B}_{R,r}^{k,m}[j] &\gets d(\mathbf{v}_{R,r}^{k,m}[j], \dot{\mathbf{v}}^{k,m}[j])
                       \end{align*}
                       \State Compute potentials:
                       \[\qquad\qquad\qquad
                       w_{r}^{k,m}[j] \coloneqq \dfrac{\pi(s_{r}^{k,m}[j])}{\pi(s_{r-1}^{a_{r-1}^k})}\cdot \dfrac{\nu^{-}(s_{r-1}^{a_{r-1}^k})}{q(s_{r}^{k,m}[j] |s_{r-1}^{a_{r-1}^k})} \]
                       \[\qquad\qquad\qquad\qquad
                       \cdot \, \text{det} \left( \dfrac{\partial \beta_{i,r}^{k,m}(\dot \bv^{k,m}[j])}{\partial \dot \bv^{k,m}[j]}\right)
                       \]
                   \EndFor
               \EndFor
               \State Extend: \[\qquad s_r^k = s_{r}^{k,i}[J] \text{ where } P(J=j, I = i) \propto w_{r}^{k,i}[j]\]
               \State Update: $w_r^k = \frac{1}{ML}\sum\limits_{l=1}^{L}\sum\limits_{m=1}^{M}w_{r}^{k,i}[j]$
           \EndFor
       \EndFor
       \State \textbf{Output:} $s_{R}^{1:K}$, $w_{1:R}^{1:K}$
   \end{algorithmic}
   \label{hncsmc1}
\end{algorithm}
\begin{figure*}[h]
    \centering
    \includegraphics[width=\linewidth]{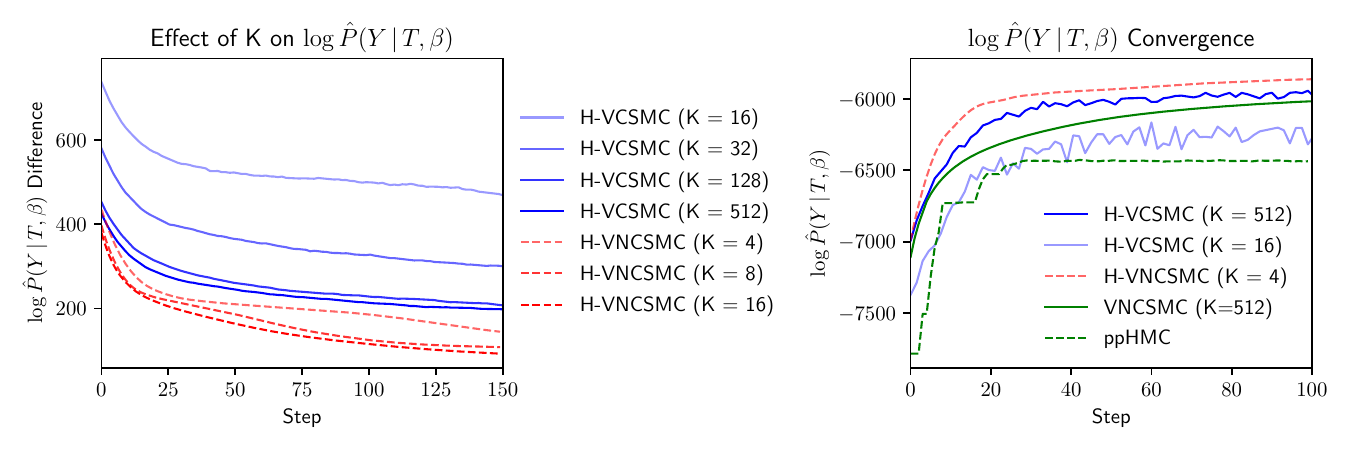}
    \vspace{-2em}
    \caption{\textbf{Left:} As \(K\) increases, the difference in log conditional likelihood \( \log \widehat{P}(Y|\tau,\beta) \) between both \textsc{H-Vcsmc} and \textsc{H-Vncsmc} and MrBayes approaches zero. MrBayes runs MCMC for 20,000 iterations.  \textbf{Right:} Convergence of \(\log \widehat{P}(Y|\tau,\beta)\) for \textsc{H-Vcsmc}, \textsc{H-Vncsmc}, \textsc{Vcsmc}, \textsc{Vncsmc}, and \textsc{ppHmc}. In particular, even at \(K=4\), \textsc{H-Vncsmc} consistently outperforms the other methods. The results are averaged over three random seeds.}
\label{primates}
\end{figure*}

\begin{proposition}
    The particles $s_r^k$ and weights $w_r^k$ generated by Algorithm~\ref{hncsmc1} are properly weighted for $\frac{\pi(s_r) \nu^-(s_{r-1})}{\pi(s_{r-1})}$.
\label{thm:pw2}
\end{proposition}
\begin{proof}
\begin{align}
    \mathbb{E}[w_r^k h(s_r^k)] &= \mathbb{E}\left[w_r^k \cdot h(s_r^{k,I}[J])\right]   
    \nonumber \\
    &= \mathbb{E}\left[\sum_{j=1}^L \sum_{i=1}^M w_r^k \frac{w_r^{k,i}[j]}{\sum_l \sum_m w_r^{k,m}[l]} h(s_r^{k,i}[j])\right] 
    \nonumber \\
    &= \frac{1}{ML} \sum_{j=1}^L \sum_{i=1}^M\mathbb{E}\left[w_r^{k,i}[j] \cdot h(s_r^{k,i}[j])\right] 
    \nonumber \\
    &= \mathbb{E}\left[w_r^{k,i}[j] \cdot h(s_r^{k,i}[j])\right] 
    \nonumber 
\end{align}
\begin{align}
    &= \sum_{\tau} \int h( \dot{\bv}_r^{k,m}[j], \tau_r^{k}[j]) 
    \dfrac{\pi(\dot \bv_{r}^{k,m}[j], \tau_r^k)\cdot \nu^{-}(\dot \bv_{r-1}^{a_{r-1}^k}, \tau_{r-1}^{a_{r-1}^k})}{\pi(\dot \bv_{r-1}^{a_{r-1}^k}, \tau_{r-1}^{a_{r-1}^k})} \nonumber \\
    &\quad \cdot \, \text{det} \left( \dfrac{\partial \beta_{i,r}^{k,m}(\dot \bv^{k,m}[j])}{\partial \dot \bv^{k,m}[j]}\right) d \dot{\bv}_r^{k,m}[j] 
    \nonumber 
\end{align}
Replacing index $i$ for index $m$ so that $i$ can be used to denote left and right branch lengths $\beta_i$:
\begin{align}
    &= \sum_{\tau} \int h(\beta_{i,r}^{k,m}[j], \tau_r^k[j]) 
    \dfrac{\pi(\beta_{i,r}^{k,m}[j], \tau_r^k) \nu^{-}(\beta_{i,r-1}^{a_{r-1}^k}, \tau_{r-1}^{a_{r-1}^k})}{\pi(\beta_{i,r-1}^{a_{r-1}^k}, \tau_{r-1}^{a_{r-1}^k})} 
    \nonumber \\
    &\quad \cdot \, \mathrm{d}\beta_{i,r}^{k,m}[j] 
    \nonumber \\
    &= \int h(s_r^k) \frac{\pi(s_r^k) \nu^-(s_{r-1}^{a_{r-1}^k})}{\pi(s_{r-1}^{a_{r-1}^k})} \, \mathrm{d}s_r^k 
    \nonumber
\end{align}
\end{proof}

\textsc{H-Ncsmc} iterates through rank events (Step 3) and Monte Carlo samples (Step 4), performing a \textsc{Resample} step (Step 5). For each sample $k$, the algorithm enumerates all $\binom{N-r}{2}$ possible one-step-ahead topologies and samples $M$ corresponding sub-branch lengths. For each of the $j= 1, \dots, \binom{N-r}{2}$ topologies, the two coalescing embedded points $\bv_{L,r}^k[j]$ and $\bv_{R,r}^k[j]$ (the left and right child nodes) are used to identify the closest point on the geodesic between them to the origin (Step 10). The parent node is then sampled in the tangent plane at the origin (Step 11) and mapped via parallel transport to its mean (Step 12). The exponential map projects these samples back onto the manifold (Step 13). The left and right branch lengths $\beta_{L,r}^{k,m}[j]$, $\beta_{R,r}^{k,m}[j]$ are computed using the geodesic distance between the sampled parent and child nodes (Step 14). Each sampled look-ahead state $s_r^{k,m}[j]$ receives a sub-weight or potential function (Step 15). The partial state for $s_r^k$ is extended by selecting one of the $\binom{N-r}{2}$ topologies and corresponding branch lengths based on their weights (Step 18). Finally, each sample's weight is computed by averaging over all sub-weights (Step 19).

We utilize the \textsc{H-Csmc} and \textsc{H-Ncsmc} algorithms to establish objective functions for \textsc{Vi}, resulting in the \textsc{H-Vcsmc} and \textsc{H-Vncsmc} methods. These algorithms are also employed to formulate variational objectives via a lower bound on the marginal likelihood estimate.

\section{Experiments}
Experiments were run on a Lambda Cloud \textsc{Gpu} instance equipped with a 30-core \textsc{Amd} \textsc{Epyc} 7J13 \textsc{Cpu}, plus an \textsc{Nvidia} A100 \textsc{Gpu} with 40 GB of VRAM. We assumed a uniform prior over topologies and a $\lambda_{bl}=10$ exponential prior over branch lengths.

\paragraph{Primates Dataset.} We applied \textsc{H-Vcsmc} and \textsc{H-Vncsmc} to the primate mitochondrial DNA dataset \citep{10.1093/oxfordjournals.molbev.a040524}, consisting of 12 taxa across 898 sites, yielding 13,749,310,575 distinct topologies. Our methods were compared to MrBayes, \textsc{ppHmc} along with reimplementations of \textsc{Vcsmc} and \textsc{Vncsmc}. In Fig.~\ref{primates} (left), we find that as \(K\) increases, the difference in log conditional likelihood \( \log \widehat{P}(Y|\tau,\beta) \) between both \textsc{H-Vcsmc} and \textsc{H-Vncsmc} and MrBayes approaches zero. MrBayes runs MCMC for 20,000 iterations. Fig.~\ref{primates}  (right) displays convergence of \(\log \widehat{P}(Y|\tau,\beta)\) for \textsc{H-Vcsmc}, \textsc{H-Vncsmc}, \textsc{Vcsmc}, \textsc{Vncsmc}, and \textsc{ppHmc}. In particular, even at \(K=4\), \textsc{H-Vncsmc} consistently outperforms the other methods. Results are averaged over three random seeds.

\begin{figure*}[h!]
    \centering
    \includegraphics[width=1.\linewidth]{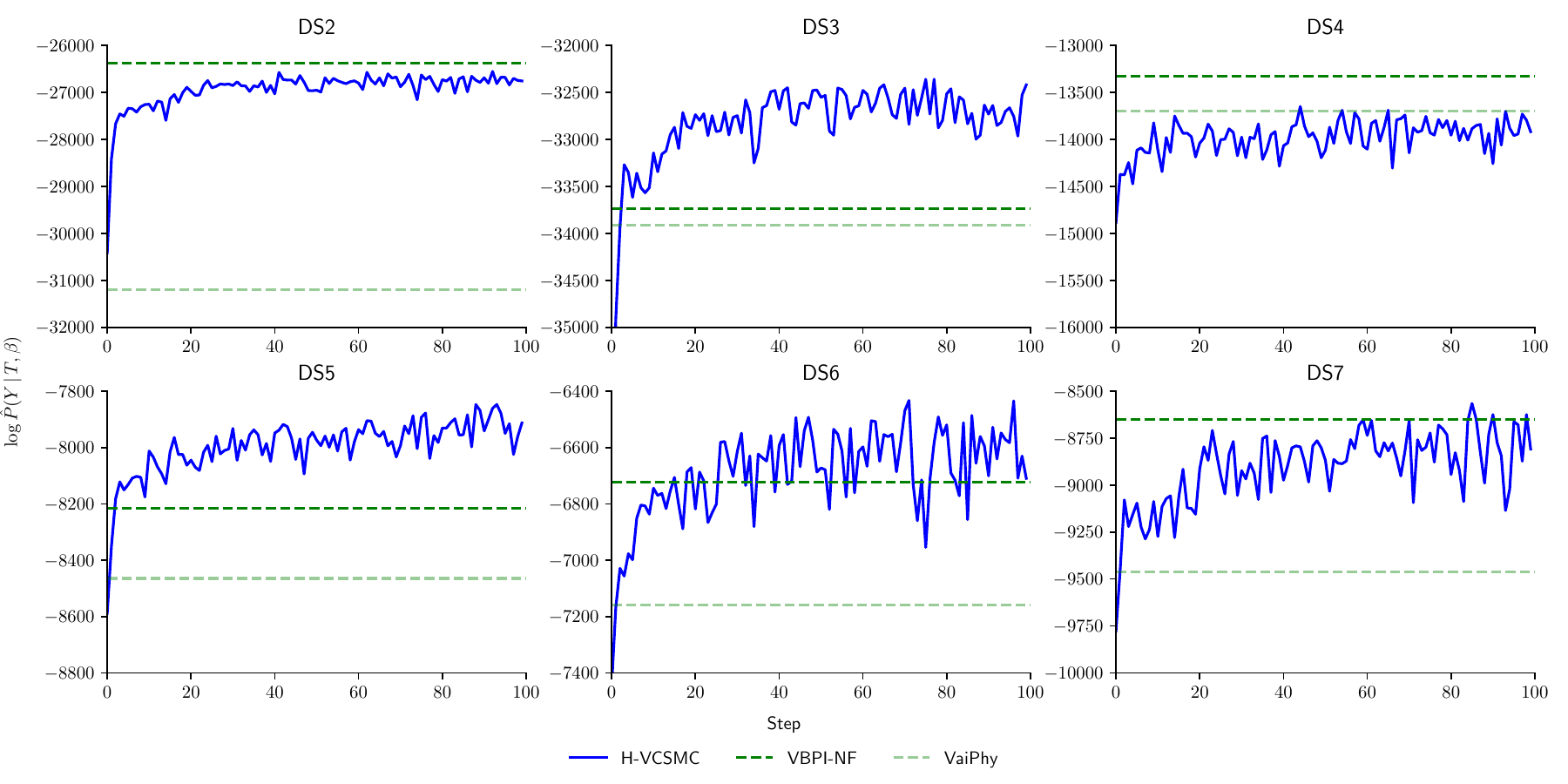}
    \caption{Convergence of \textsc{H-Vcsmc} on 6 out of 7 benchmark datasets. \textsc{H-Vcsmc} often reaches higher values than baseline methods prior to convergence, indicating greater efficiency.
}
\end{figure*}

\begin{table*}[h!]
\centering
\caption{Comparison of marginal log likelihood estimates across 7 datasets using various inference techniques.}
\scalebox{.75}{
\begin{tabular}{lcccccccccc}
\toprule
\textbf{Dataset} & \textbf{H-VCSMC} & \textbf{H-VNCSMC} & \textbf{VBPI-NF} & \textbf{MrBayes} & \textbf{Vaiphy} & \multicolumn{2}{c}{\textbf{VCSMC-PyTorch}} & \textbf{GeoPhy} & \textbf{Dodonophy} \\
& & & & & & \textbf{K=512} & \textbf{K=1024} & & \\
\midrule
DS1 & -7639.21 & -7290.18 & -7108.40 & \textbf{-7032.45} & -7490.54 & -7728.43 & -7669.21 & -7111.70 & -8042.10 \\
DS2 & -26799.23 & -26500.64 & -26367.70 & \textbf{-26363.85} & -31203.44 & -27133.01 & -27017.15 & -26372.84 & -26777.43 \\
DS3 & -32535.99 & \textbf{-31818.26} & -33735.09 & -33729.60 & -33911.13 & -33121.44 & -32851.80 & -33766.49 & -34437.62 \\
DS4 & -13954.06 & -13849.76 & -13329.93 & \textbf{-13292.37} & -13700.86 & -14212.81 & -14119.78 & -13343.06 & -15070.36 \\
DS5 & -8141.26 & \textbf{-8012.42} & -8214.61 & -8192.96 & -8464.77 & -8402.81 & -8367.93 & -8233.22 & -13702.80 \\
DS6 & -6960.84 & -6789.01 & -6724.36 & \textbf{-6571.02} & -7157.84 & -6976.13 & -6888.14 & -6734.07 & -9595.49 \\
DS7 & -9116.36 & -8922.34 & -8650.49 & \textbf{-8438.78} & -9462.21 & -9509.71 & -9493.81 & -8658.93 & -- \\
\bottomrule
\end{tabular}
}
\label{table:comparison1}
\end{table*}

\begin{table*}[h!]
\centering
\caption{Running Times in Decimal Hours for Various Datasets Across Different Methods.}
\scalebox{.75}{
\begin{tabular}{lccccccc}
\toprule
\textbf{Dataset} & \textbf{GeoPhy} & \textbf{VCSMC Original} & \textbf{VCSMC Original} & \textbf{VCSMC K=512} & \textbf{VCSMC K=1024} & \textbf{H-VCSMC} & \textbf{MrBayes} \\
 & & \textbf{K=512} & \textbf{K=1024} & \textbf{(PyTorch)} & \textbf{(PyTorch)} & \textbf{K=512} & \\
\midrule
DS1 & 6.78 & 1.07 & 2.25 & 0.28 & 0.45 & 0.16 & 1.864 \\
DS2 & 6.69 & 1.57 & 3.27 & 0.25 & 0.48 & 0.21 & 2.8275 \\
DS3 & 8.08 & 1.85 & 3.77 & 0.23 & 0.44 & 0.21 & 3.0277 \\
DS4 & 7.67 & 1.40 & 2.88 & 0.18 & 0.33 & 0.19 & 1.9722 \\
DS5 & 8.67 & 0.55 & 1.17 & 0.09 & 0.16 & 0.15 & 0.9138 \\
DS6 & 8.20 & 2.08 & 4.20 & 0.23 & 0.41 & 0.24 & 1.4638 \\
DS7 & 8.88 & 2.63 & 5.28 & 0.26 & 0.65 & 0.31 & 1.5194 \\
\bottomrule
\end{tabular}
}
\label{table:running_times}
\end{table*}

\paragraph{Large Taxa Benchmarks.} We evaluated \textsc{H-Vcsmc} on seven large phylogenetic benchmark datasets, ranging from 27 to 64 eukaryote species and spanning 378 to 2520 sites~\citep{10.1093/oxfordjournals.molbev.a040628, Garey1996, 10.1080/10635150390235557, doi:10.1080/15572536.2004.11833059, lakner, doi:10.1080/00275514.2001.12063167,  doi:10.1080/00275514.2001.12061283}. Table \ref{table:comparison1} compares log conditional likelihood values across several methods. \textsc{H-Vcsmc} and \textsc{H-Vncsmc} show competitive performance compared to existing approaches. \textsc{Vbpi-Nf} uses precomputed tree topology support and extensive bootstrap sampling. While MrBayes achieves the highest likelihood on most datasets, it requires 4.9$\times- 14.4\times$ longer running times (see Table \ref{table:running_times}). Notably, \textsc{H-Vcsmc} demonstrates efficient convergence, running 40$\times-60\times$ faster than GeoPhy while achieving higher likelihood values on two datasets and closely matching its performance on another. 
Results use MrBayes Stepping Stone sampling with four chains for 10,000,000 iterations, sampling every 100 iterations, following VaiPhy's default settings.

\paragraph{Implementation Details.}
Algorithm~\ref{sample_parent_embeddings} incurs minimal overhead, with complexity unchanged from the original methods except for embedding points in the Poincaré disk. The \textsc{Ncsmc} algorithm increases computational cost by marginalizing intermediate target densities to improve partial state exploration. While \textsc{H-Csmc} runs in $\mathcal{O}(KNM)$, \textsc{H-Ncsmc} requires $\mathcal{O}(KN^3M)$. The fast execution of \textsc{H-Vcsmc} is due to an efficient PyTorch implementation, fully vectorized across particles and sites except for resampling. 

\section{Conclusion} 
We have introduced hyperbolic extensions to the Combinatorial and Nested Combinatorial Sequential Monte Carlo algorithms, leading to the development of two hyperbolic \textsc{Vi} methods. Both \textsc{H-Vcsmc} and \textsc{H-Vncsmc} demonstrate strong empirical performance, significantly improving both scalability and effectiveness in high-dimensional inference. Our results show that these techniques offer a powerful and computationally efficient alternative for phylogenetic modeling tasks. An implementaiton of both \textsc{H-Vcsmc} and \textsc{H-Vncsmc} along with a high-performance PyTorch reimplementation of \textsc{Vcsmc} is available at \url{https://github.com/axchen7/vcsmc}.

\vspace{-.5em}
\paragraph{Acknlowledgements}
We acknowledge funding from the Atlanta University Consortium Data Science Initiative. 

\clearpage
\bibliography{iclr2025_conference}
\bibliographystyle{iclr2025_conference}

\clearpage
\onecolumn
\appendix
\section{Appendix}

\subsection{Poincaré Disk Model}
The \textit{Poincaré disk model} represents hyperbolic geometry within the unit disk in Euclidean space. This model is defined on the open unit disk, denoted as 
\begin{equation}
    \mathcal{P} \coloneqq \left\{ \bz \in \mathbb{R}^2 : \|\bz\| < 1 \right\} \equiv \{\bz \in \mathbb{C}: |\bz| < 1\}  \,.
\end{equation}
Let \( \|\bz\| \) denote the Euclidean norm of the point \( \bz \). The hyperbolic distance between two points \( \bz_1 \) and \( \bz_2 \) inside the unit disk is given by:
\begin{equation}
    d(\bz_1, \bz_2) = \text{arcosh}\left(1 + \frac{2 \|\bz_1 - \bz_2\|^2}{(1 - \|\bz_1\|^2)(1 - \|\bz_2\|^2)}\right)\,.
    \label{poincare_distances}
\end{equation}

\subsubsection{Metric Tensor}
The \textit{metric tensor} in the Poincaré disk is defined using a conformal modification of the Euclidean metric:
\begin{equation}
    ds^2 = \frac{4(dx^2 + dy^2)}{(1-x^2-y^2)^2}\,,
\end{equation}
where $dx^2 + dy^2$ is the standard Euclidean metric in $\mathbb{R}^2$ and $4/(1-x^2-y^2)^2$ is a conformal scaling factor dependent upon distance from the origin. As points approach the boundary of the disk ($x^2 + y^2 \rightarrow 1$), distances tend to infinity and the metric becomes increasingly large. 

\subsubsection{Geodesics in the Poincaré Disk}
\label{geodesics}
\begin{figure}
\begin{tikzpicture}[scale=3]

    \draw[thick, black] (0,0) circle [radius=1];
    \draw[thick, black] (0.4583333333333333, 1.0416666666666665) circle [radius=0.5432668671002207];
    
    \coordinate (P) at (0,0.75);
    \coordinate (Q) at (0.5,0.5);
    \coordinate (C) at (0,0);
    \coordinate (Pinv) at (0.0, 1.3333333333333333);
    \coordinate (Qinv) at (1.0, 1.0);
    \coordinate (CGeo) at (0.46,1.04);
    \coordinate (Pinter) at (-0.08305646147470835, 0.9965448430488716);
    \coordinate (Qinter) at (0.7908312604023223, 0.612034245422978);    
    \def\radius{0.54}
    
    \draw[red] (P) arc[start angle=-147.53, end angle=-85.60, radius=\radius];
    
    \draw[fill=red, draw=black] (P) circle (1pt);
    \draw[fill=red, draw=black] (Q) circle (1pt);
    \draw[fill=blue, draw=black] (C) circle (1pt);
    \draw[fill=purple, draw=black] (Pinv) circle (1pt);
    \draw[fill=purple, draw=black] (Qinv) circle (1pt);
    \draw[fill=brown, draw=black] (Pinter) circle (0.5pt);
    \draw[fill=brown, draw=black] (Qinter) circle (0.5pt);
    \draw[dash pattern=on 0.45pt off 1pt] (C)--(Pinv);
    \draw[dash pattern=on 0.45pt off 1pt] (C)--(Pinter);
    \draw[dash pattern=on 0.45pt off 1pt] (C)--(Qinv);
    \draw[dash pattern=on 0.45pt off 1pt] (C)--(Qinter);

    \node[above] at (-0.2, 1.3) {$P^{-1}$};
    \node[above] at (1.2, 1.) {$Q^{-1}$};
    \node[above] at (-0.2, .6) {$P$};
    \node[above] at (.6, .2) {$Q$};
    
\end{tikzpicture}
\centering
\label{fig:poincare_geodesics}
\caption{Geodesic arc in the Poincaré disk connecting points \( P \) and \( Q \) depicted in red, with their corresponding inverses \( P^{-1} \) and \( Q^{-1} \). A circle orthogonal to the unit circle that passes through points \( P \) and \( Q \) must also intersect their inverses, \( P^{-1} \) and \( Q^{-1} \), which are the reflections of \( P \) and \( Q \) with respect to the unit circle.}
\end{figure}
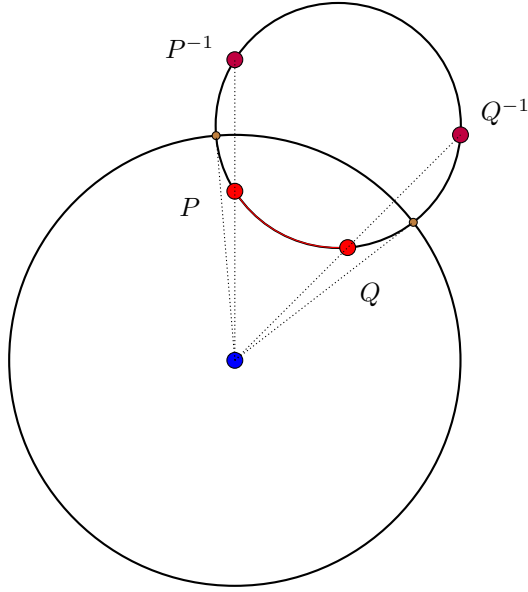
In the Poincaré disk model, geodesics are represented by arcs of circles that meet the boundary of the disk at right angles or by diameters of the disk, as shown in Fig.~\ref{fig:poincare_geodesics}. According to the Tangent-Secant Theorem, any circle orthogonal to the unit circle that passes through the points \(P\) and \(Q\) must also intersect their circular inverses, \(P^{-1}\) and \(Q^{-1}\), which are reflections with respect to the unit circle. As a result, the hyperbolic line between \(P = (x_1, y_1)\) and \(Q = (x_2, y_2)\) corresponds to an arc of the circle that passes through \(P\), \(Q\), and
\begin{equation}
    P^{-1} = \left( \frac{x_1}{x_1^2 + y_1^2}\,, \,\frac{y_1}{x_1^2 + y_1^2} \right)\,.
\end{equation}
To plot the geodesic arc, find the center and radius of the circle passing through \(P\), \(Q\), and \(P^{-1}\). This can be done by mapping each coordinate pair \((x, y)\) to the complex plane as \(z = x + iy\), and applying the linear transformation:
\begin{equation}
z \rightarrow \frac{z_3 - z_1}{z_2 - z_1}\,,
\end{equation}
which maps the points \((z_1, z_2, z_3)\) to new coordinates:
\begin{equation}
\left(0, 1, w = \frac{z_3 - z_1}{z_2 - z_1}\right)\,.
\end{equation}
Imposing the condition that each transformed point satisfies the equation for a circle \( |z - c| = r \), derive the system of equations:
\begin{align}
|c|^2 &= r^2 \\
1 - c - \bar{c} + |c|^2 &= r^2 \\
|w|^2 - \bar{w} c - w \bar{c} + |c|^2 &= r^2 
\end{align}
Solving this system yields the center of the transformed circle as:
\begin{equation}
c' = \frac{w - |w|^2}{w - \bar{w}}\,.
\end{equation}
To revert this transformation back to the original coordinates, the center becomes:
\begin{equation}
c = (z_2 - z_1)\frac{w - |w|^2}{w - \bar{w}} + z_1\,.
\end{equation}
Finally, the radius is determined by the relation:
\begin{equation}
r = |z_1 - c|\,.
\end{equation}

\subsubsection{Parallel Transport and Exponential Map}
\label{parallel_transport_and_expmap}
Hyperbolic space is not a vector space, which renders the linear operations typically employed in vector spaces inapplicable. Riemannian geometry requires special tools to understand how vectors change as they move across curved surfaces. \textit{Parallel transport} adjusts vectors along a curve while preserving their relationship to the surface's curvature, ensuring consistent comparisons between different points. The \textit{exponential map} projects directions from the tangent space to the manifold along geodesics, allowing us to translate local, linear information into motion across the curved space.

\paragraph{Parallel transport.} Given two vectors $\bx, \by \in \mathcal{P}$, the parallel transport from $\bx$ to $\by$ is defined as a map $\text{PT}_{\bx \to \by}$ from the tangent plane at $\bx$, $T_{\bx}\mathcal{P}$, to the tangent plane at $\by$, $T_{\by}\mathcal{P}$. This map transports a vector in $T_{\bx}\mathcal{P}$ along the geodesic from $\bx$ to $\by$ in a parallel manner, preserving the metric tensor throughout the process. When $\bx$ is the origin $\mu_0$, \cite{NEURIPS2018_dbab2adc} shows that parallel transport is defined as: 
\begin{equation}
    PT_{\mu_0 \rightarrow \by}(\bv) \coloneqq \frac{2}{\lambda_\by^K}\bv\,,
\end{equation}
where $\bv$ denotes the vector being transported from tangent plane at the origin $\mu_0$ to the tangent plane at point $\by$ in the Poincare disk and $\lambda_\by^K$ is the conformal factor defined below:
\begin{equation}
    \lambda_\by^K \coloneqq \frac{2}{1+K||\bv||_2^2}\,.
\end{equation}
In the Poincaré disk, $K=-1$ so that the conformal factor simplifies to:
\begin{equation}
    \lambda_\by= \frac{2}{1 - ||\by||_2^2}\,,
\end{equation}
and parallel transport simplifies to:
\begin{equation}
    PT_{\mu_0 \rightarrow \by}(\bv) = (1-||\by||_2^2)\bv\,.
\end{equation}

\paragraph{Möbius addition.}
In hyperbolic space, Euclidean vector addition does not account for the curvature that affects distances and angles. Möbius addition is an operation for combining points taking into account the negative curvature of the manifold.
Möbius addition \( \oplus_K \) for \( x, y \in \mathbb{M}_K \) (considering both signs of \( K \)) is defined as
\begin{equation}
    \bx \oplus_K \by \coloneqq \frac{(1 - 2K \langle \bx, \by \rangle - K ||\by||_2^2)\bx + (1 + K||\bx||_2^2)\by}{1 - 2K\langle \bx, \by \rangle  + K^2||\bx||^2_2||\by||^2_2 }\,.
\end{equation}
which simplifies when $K = -1$ to the following inside the Poincaré  disk:
\begin{equation}
    \bx \oplus \by = \frac{(1 + 2 \langle \bx, \by \rangle + ||\by||_2^2)\bx + (1 - ||\bx||_2^2)\by}{1 + 2\langle \bx, \by \rangle  + ||\bx||^2_2||\by||^2_2 }\,.
\end{equation}
For additional details see \cite{ungar2009gyrovector}.

\paragraph{Exponential and logarithmic map.} 
The exponential map converts directions from the flat tangent plane into movement along the geodesics of the curved surface. The exponential map is defined as:
\begin{equation}
\exp_\bx^K(\bv) = \bx \oplus_K \left( \tanh \left( \sqrt{-K} \frac{\lambda_x^K ||\bv||_2}{2} \right) \frac{\bv}{\sqrt{-K}||\bv||_2} \right)\,
\end{equation}
which simplifies to:
\begin{equation}
\exp_x(\bv) = x \oplus \left( \tanh \left( \frac{\lambda_x^{-1} ||\bv||_2}{2} \right) \frac{\bv}{||\bv||_2} \right)\,.
\end{equation}
Recall that:
\begin{equation}
\lambda_\bx = \frac{2}{1 - ||\bx||_2^2}\,.
\end{equation}
Therefore, we have:
\begin{equation}
\exp_x(\bv) = \bx \oplus \left( \tanh \left( (1 - ||\bx||_2^2)||\bv||_2 \right) \frac{\bv}{||\bv||_2} \right)\,.
\end{equation}
The logarithmic map, which is the inverse of the exponential map, is defined in the Poincaré disk as follows:
\begin{equation}
    \log_{\bx}(\by) \coloneqq \tanh^{-1} \left(\|\by\|_2\right)\frac{\by}{\|\by\|_2}\,.
\end{equation}

\subsubsection{Circles in the Poincar\'e Disk}
In this section, we will derive a basic formula for converting between the hyperbolic center $
\mathbf{c}$ and radius $r$ of a circle and the corresponding Euclidean center $\mathbf{c}'$ and radius $r'$.
This is useful for plotting and for the subsequent branch length derivation in Section~\ref{branch_lengths}.
For a visualization of the relationship between $\mathbf{c}, \mathbf{c'}, r$, and $r'$, see the left side of Figure~\ref{fig:twocircles}.

Without loss of generality, let us assume that $\mathbf{c}$ lies on the $x$-axis. Throughout this derivation, we will use boldface (e.g. $\mathbf{c}$) to refer to points, and $c$ to refer to just their $x$-coordinates.\footnote{Our derivation is most simply expressed for points collinear on the line $y=0$. To transform this to an arbitrary line, observe that the $x$-coordinates are also \textit{norms} of vectors; thus, they can be used to multiply a unit vector pointing from the origin of the Poincar\'e disk to the (hyperbolic) center of our circle $\mathbf{c}$.} Our goal is to find $x_l$ and $x_r$ such that $d(\mathbf{c}, \mathbf{x_l}) = d(\mathbf{c}, \mathbf{x_r}) = r$ and $x_l < c < x_r$, and use these diametrically-opposed points to compute a Euclidean center and radius.

We do so as follows:
\begin{align}
    d(\mathbf{c}, \mathbf{x_l}) = r &= \text{arcosh}\left( 1 + \frac{2(\|\mathbf{c} - \mathbf{x_l}\|^2)}{(1-\|\mathbf{x_l}\|^2)(1 - \|\mathbf{c}\|^2} \right)\\
    r &= \text{arcosh}\left( 1 + \frac{2(c - x_l)^2}{(1-x_l^2)(1 - c^2)} \right)\\
    \cosh(r) - 1 &= \frac{2(c - x_l)^2}{(1-x_l^2)(1 - c^2)}
\end{align}
Letting $u = \cosh(r) - 1$ and $v = 1 - c^2$:
\begin{align}
    \frac{2(c^2 + x_l^2 - 2cx_l)}{1 - x_l^2} &= uv\\
    2c^2 + 2x_l^2 - 4cx_l &= uv - uvx_l^2\\
    \underbrace{(2+uv)}_{A}x_l^2 + \underbrace{(-4c)}_{B}x_l + \underbrace{(2c^2 - uv)}_{C} &= 0
\end{align}
By the quadratic formula:
\begin{align}
    x_l &= \frac{-B - \sqrt{B^2 - 4AC}}{2A}\\
    x_r &= \frac{-B + \sqrt{B^2 - 4AC}}{2A}\\
    c' &= \frac{x_l + x_r}{2} = \frac{-B}{2A} = \frac{2c}{2 + (\cosh(r)-1)(1 - c^2)}
\end{align}
Thus, for any $c \in [-1, 1]$,
\begin{equation}
    c' = \frac{c}{1 + \frac{1}{2}(\cosh(r)-1)(1-c^2)}
\end{equation}
The Euclidean radius is simply:
\begin{align}
    r' = \frac{x_r - x_l}{2}.
\end{align}

\begin{figure}[t]
   \centering
   \begin{tikzpicture}[scale=3]
       \draw[thick] (0,0) circle (1cm);
       \draw[thick, color=blue] (.3333, 0) circle (0.6547);
       \node[above] at (0.8038, 0) {$\mathbf{c}$};
       \node[above] at (.3333, 0) {$\mathbf{c}'$};
       \draw[thick, color=teal] (.3333, 0) -- (.3333, -0.6547);
       \node[left] at (.3333, -0.3274) {$r'$};
       \draw[thick, color=red] (.3333, -0.6547) arc (180:108.5875:0.6906);
       \node[right] at (.45, -0.3274) {$r$};
       \node[above] at (-.4, 0) {$\mathbf{x_l}$};
       \node[above] at (1.1, 0) {$\mathbf{x_r}$};
       \draw[fill=black, draw=black] (0.8038, 0) circle (.7pt);
       \draw[fill=black, draw=black] (.3333, 0) circle (.7pt);
       \draw[fill=black, draw=black] (-.3214, 0) circle (.7pt);
       \draw[fill=black, draw=black] (.988, 0) circle (.7pt);
   \end{tikzpicture}
   \hspace{1cm}
   \begin{tikzpicture}[scale=3]
       \draw[thick] (0,0) circle (1cm);
       \draw[thick, color=blue] (0.3822, -0.3196) circle (0.4796);
       \node[right] at (0.6227, -0.5207) {$R$};
       \draw[thick, color=blue] (-0.0190, 0.3600) circle (0.6300);
       \node[left] at (-0.0439, 0.8321) {$L$};
       \node[right] at (0.58, 0.14) {$\alpha$};
       \node[left] at (-0.1, -0.35) {$\tilde{\alpha}$};
       \draw[thick, color=red] (0.6227, -0.5207) arc (224.9637:187.5004:2.3481);
       \node[above] at (0.2, 0.13) {$\dot{\mathbf{v}}_r^k$};
       \draw[thick, color=violet] (0.1826, 0.0910) arc (273.0392:276.7658:5.9020);
       \draw[thick, color=violet] (0.1826, 0.0910) arc (139.9687:144.3570:5.8952);
       \fill (0.6227, -0.5207) circle (.7pt);
       \fill (-0.0439, 0.8321) circle (.7pt); 
       \fill (0.5650, 0.1238) circle (.7pt);
       \fill (-0.0943, -0.2655) circle (.7pt);
       \fill (0.1826, 0.0910) circle (.7pt);
   \end{tikzpicture}
   \caption{\textbf{Left: } Hyperbolic circles are also Euclidean circles (blue): the Euclidean distance from any point on the circle to $\mathbf{c'}$ is the Euclidean radius $r'$ (teal), and the hyperbolic distance from any point on the circle to $\mathbf{c}$ is the hyperbolic radius $r$ (red). \textbf{Right: }To visualize how branch lengths $(\beta_{R,r}^k, \beta_{L,r}^k)$ are jointly distributed for coalescing nodes $R$ and $L$, we plot circles (blue) with corresponding hyperbolic radii centered at each node. This shows that any pair of branch lengths corresponds to at most two points in the Poincaré disk. In the above, $\dot{\mathbf{v}}_r^k$ is the mean ancestor node, i.e. the point on the geodesic between $L$ and $R$ (red) closest to the origin. The likelihood of drawing either of the intersection points $\alpha$ or $\tilde{\alpha}$ is depends on their geodesic distances to $\dot{\mathbf{v}}_r^k$ (violet). In practice, since the branch lengths are computed by sampling from a $\mathcal{WN}$ centered at $\dot \bv_r^k$, one of the intersections is typically much closer than the other.}
   \label{fig:twocircles}
\end{figure}
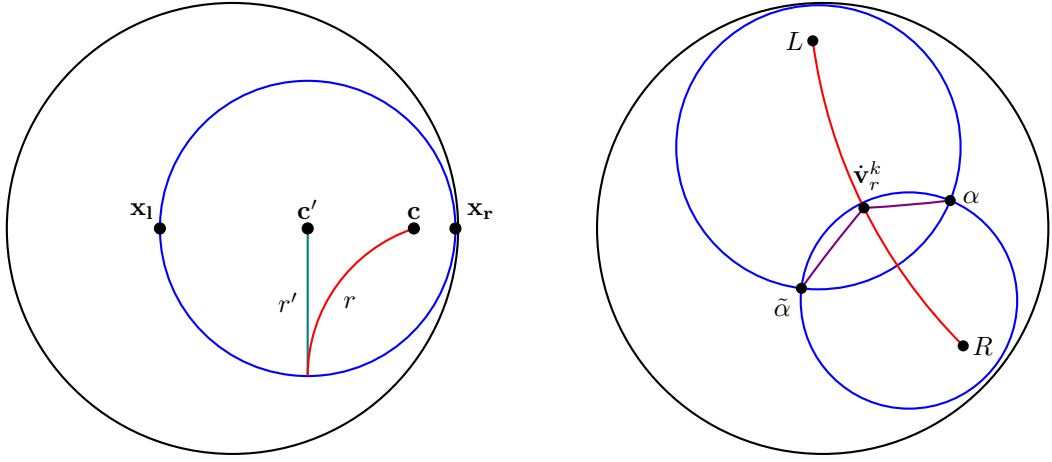

\clearpage
\subsection{Combinatorial Sequential Monte Carlo}
\label{csmc_appendix}
We review several details of the \textsc{Csmc} algorithm. Additional details can be found in~\cite{csmc}.

\subsubsection{Partial States and the Natural Forest Extension}
\label{partial_states}
\begin{definition}[Partial State] A partial state of rank $r \in {0, \cdots, N-1}$, denoted by $s_r = (t_i, X_i)$, represents a collection of rooted trees and must satisfy three key conditions: \begin{enumerate} \item Partial states at different ranks do not overlap, meaning for any distinct ranks $r$ and $s$, their corresponding sets of partial states are disjoint, i.e., $\forall r \neq s$, $\mathcal{S}_r \cap \mathcal{S}_s = \emptyset$. \item The partial state set at the lowest rank consists of a single element, denoted as $S_0 = { \bot }$. \item The partial states at the highest rank, $R = N-1$, correspond to the target space $\mathcal{X}$. \end{enumerate} \end{definition}

The likelihood and the measure $\pi$ are specifically defined for the target space, with $\mathcal{S}_R = \mathcal{X}$. These definitions apply solely to the target space of trees and do not extend to the broader space of partial states $\mathcal{S}_{r<R}$, which is made up of forests with disjoint trees. Felsenstein's pruning algorithm \citep{Felsenstein:1981:J-Mol-Evol:7288891} is  used to score the likelihood of a tree. However, partial states are explicitly collections of disjoint trees or leaf nodes. To extend the target measure $\pi$ to include the partial state space $\mathcal{S}_{r<R}$, one method is to treat each disjoint tree in the forest independently, as described in~\cite{csmc}.

\begin{definition}[Natural Forest Extension] The natural forest extension (\textsc{Nfe}) expands the target measure $\pi$ to include forests by taking the product over the trees that compose the forest: \begin{equation} \pi(s) \coloneqq \prod\limits_{(t_i,X_i)}^{} \pi_{Y_i(x_i)}(t_i),. \end{equation} \end{definition}

A key feature of the \textsc{Nfe} is its ability to incorporate information from non-coalescing elements into the local weight updates.

\subsubsection{Encoding Leaf Nodes in the Poincare Disk}
\begin{definition}[\(\delta\)-Hyperbolicity] 
A geodesic metric space \( \mathcal{X} \) is said to be \emph{\(\delta\)-hyperbolic} if there exists a constant \( \delta \geq 0 \) such that for any four points \( \bw, \bx, \by, \bz \in \mathcal{X}\), the following holds:
\begin{equation}
    d(\bw,\bx) + d(\by,\bz) \leq \max\{ d(\bx,\by) + d(\bw,\bz)\,,\, d(\bx,\bz) + d(\bw,\by)\} + 2\delta
\end{equation}
It is well established that trees are \(0-\)hyperbolic ~\citep{gromov1987hyperbolic}, which aligns with the four-point condition used in Unweighted Pair Group Method with Arithmetic Mean (UPGMA)~\citep{sokal1958statistical}. Intuitively, if the embedding is perfect, the distance between any two leaves in hyperbolic space should correspond to their Hamming distance. In cases of perfect ultrametricity, the points should approach the boundary of the Poincaré Disk. This alignment reflects the structure of the tree, where the distance from the origin can be interpreted as a molecular clock and represents the time until the divergence of organisms.

To further understand the implications of this structure, consider the four-point condition, which states that for any set of four leaves, there are three ways to group them into unordered pairs. The grouping that pairs closely related leaves should yield the minimal sum of pairwise distances, while the other two groupings—representing distances across subtrees—should be approximately equal. In hyperbolic spaces, this behavior is characterized by \(\delta\)-hyperbolicity, which asserts that the deviation between the two maximal pairwise distances is bounded by some \(\delta\), and decreases as points move closer to the boundary. This relationship emphasizes how the molecular clock and the properties of hyperbolic space work in tandem to inform our understanding of evolutionary relationships.

\end{definition}


\subsection{Branch Length Probability Derivation}
\label{branch_lengths}


The mapping between embeddings and branch lengths is not injective. To see this, consider fixing two branch lengths, $\beta_{L,r}^k$ and $\beta_{R,r}^k$, and points for the left and right child nodes, $L$ and $R$, on the disk. The points that are a geodesic distance $\beta_{L,r}^k$ from $L$ form the circumference of a circle. Similarly, the points at a distance $\beta_{R,r}^k$ from $R$ form the circumference of another circle. The intersection of these two circumferences can result in two possible locations for the parent point, as illustrated in the right side of Fig.~\ref{fig:twocircles}. To take into account this two to one mapping, the density of branch lengths will be a sum of two terms, one for each intersection point ($\alpha$ and $\tilde \alpha$) on the disk:
\begin{equation}
     q_\psi(\dot{\mathbf{v}}_r^k \mid \dot{\mathbf{v}}_{r-1}^{a_{r-1}^k}, \tau_{r-1}^{a_{r-1}^k}) = q_\psi(\dot{\mathbf{v}}_r^k \mid \dot{\mathbf{v}}_{r-1}^{a_{r-1}^k} = \alpha, \tau_{r-1}^{a_{r-1}^k}) + q_\psi(\dot{\mathbf{v}}_r^k \mid \dot{\mathbf{v}}_{r-1}^{a_{r-1}^k} = \tilde{\alpha}, \tau_{r-1}^{a_{r-1}^k})
\end{equation}
Due to concentration of measure, the further solution maps to a point where the $\mathcal{WN}$ density is exponentially smaller than the closer point.

\subsubsection{Jacobian Derivation}
Given the distance function between two points \( \bz_1 \) and \( \bz_2 \) in hyperbolic space:
\begin{equation}
    d(\bz_1, \bz_2) = \text{arcosh}\left(1 + \frac{2 \|\bz_1 - \bz_2\|^2}{(1 - \|\bz_1\|^2)(1 - \|\bz_2\|^2)}\right)
\end{equation}
The derivative of \( \text{arcosh}(x) \) is given by:
\begin{equation}
    \frac{d}{dx} \text{arcosh}(x) = \frac{1}{\sqrt{x^2 - 1}}
\end{equation}
Let:
\begin{equation}
    f(\bz_1, \bz_2) = 1 + \frac{2 \|\bz_1 - \bz_2\|^2}{(1 - \|\bz_1\|^2)(1 - \|\bz_2\|^2)}
\end{equation}
We can now apply the chain rule to compute the derivative of \( d(\bz_1, \bz_2) \) with respect to \( \bz_1 \):
\begin{equation}
    \frac{\partial d(\bz_1, \bz_2)}{\partial \bz_1} = \frac{1}{\sqrt{f(\bz_1, \bz_2)^2 - 1}} \cdot \frac{\partial f(\bz_1, \bz_2)}{\partial \bz_1}
\end{equation}
The derivative of \( f(\bz_1, \bz_2) \) with respect to \( \bz_1 \) is given by:
\begin{equation}
    \frac{\partial f(\bz_1, \bz_2)}{\partial \bz_1} = \frac{2 \left( \bz_1 - \bz_2 \right)}{(1 - \|\bz_1\|^2)(1 - \|\bz_2\|^2)} - \frac{4 \|\bz_1 - \bz_2\|^2 \bz_1}{(1 - \|\bz_1\|^4)(1 - \|\bz_2\|^2)}
\end{equation}
The Jacobian matrix \( J_{d(\bz_1, \bz_2)} \) is formed from the partial derivatives of \( d(\bz_1, \bz_2) \) with respect to each component of \( \bz_1 \). The determinant of this Jacobian matrix is given by:
\begin{equation}
    \text{det}(J_{d(\bz_1, \bz_2)}) = \text{det} \left( \frac{\partial}{\partial \bz_1} \left( \frac{1}{\sqrt{f(\bz_1, \bz_2)^2 - 1}} \cdot \frac{\partial f(\bz_1, \bz_2)}{\partial \bz_1} \right) \right)
\end{equation}
Since \( \dot \bv_{r-1}^{a_{r-1}^k} \) can take two possible values ($\alpha$ and $\tilde \alpha$), the Jacobian will capture how the change in \( \dot \bv_{r-1}^{a_{r-1}^k} \) (between $\alpha$ and $\tilde \alpha$) affects the total proposal expression.
Thus, the Jacobian can be written as the sum of two partial derivatives, corresponding to each term in the expression:
\begin{equation}
    \mathbf{J} = \dfrac{\partial \beta_{i,r}^k(\dot \bv^k)}{\partial \dot \bv^k} \coloneqq \frac{\partial}{\partial \dot{\mathbf{v}}_{r-1}^{a_{r-1}^k}} \left( q_\psi(\dot{\mathbf{v}}_r^k \mid \dot{\mathbf{v}}_{r-1}^{a_{r-1}^k} = \alpha, \tau_{r-1}^{a_{r-1}^k}) \right) + \frac{\partial}{\partial \dot{\mathbf{v}}_{r-1}^{a_{r-1}^k}} \left( q_\psi(\dot{\mathbf{v}}_r^k \mid \dot{\mathbf{v}}_{r-1}^{a_{r-1}^k} = \tilde{\alpha}, \tau_{r-1}^{a_{r-1}^k}) \right)
\end{equation}
In our implementation, we define the Jacobian via automatic differentiation in PyTorch. 

\clearpage

\subsection{Hyperbolic CSMC}
\label{hyperboliccsmc}
\textsc{H-Csmc} is described in detail in Algorithm \ref{alg:hcsmc}.

\begin{algorithm}[h!]
    \caption{Combinatorial Sequential Monte Carlo with Poincaré Embeddings}
    \label{alg:hcsmc}
    \begin{algorithmic}[1]
        \State \textbf{Input:} $\mathbf{Y} = \{Y_1,\cdots,Y_N \} \in \Omega^{N \times M}$,  $\theta = (\mathbf{Q},\{\lambda_i\}_{i=1}^{|E|})$
        \State \textbf{Embed sequences in $\mathcal{P}$:} $\forall n, Y_n \rightarrow f(Y_n)$ where $f: \Omega \rightarrow \mathcal{P}$
        \State \textbf{Initialization:} $\forall k$, $s_{0}^k \leftarrow \perp$, $w_{0}^k \leftarrow 1/K$
        \For{$r = 0$ \textbf{to} $R = N-1$}
            \For{$k = 1$ \textbf{to} $K$}
                \State \textsc{Resample:}
                \[
                \mathbb{P}(a_{r-1}^k = i) = \frac{w_{r-1}^i}{\sum_{l=1}^{K} w_{r-1}^l}
                \]
                \State \textsc{Extend partial state:}
                \vspace{1em}
                
                \State \hspace{1em} (a) Sample two embedded points $\mathbf{v}_{L,r}^k$, $\mathbf{v}_{R,r}^k$ to coalesce uniformly.
                \[
                    \mathbb{P}(\bv_{L,r}^k = i, \bv_{R,r}^k = j) = \frac{1}{\binom{N-r+1}{2}}
                \]
                \State \hspace{1em} (b) Find the point on the geodesic arc between $\mathbf{v}_{L,r}$ and  $\mathbf{v}_{R,r}$ closest to the origin:
                \[
                \bv_p \gets \argmin_{\bv \in \gamma(\bv_{L,r}, \bv_{R,r})} 
                d(\mu_0, \bv)
                \]
                \State \hspace{1em} (c) Sample a Gaussian in $\mathbb{R}^2$ to yield a vector in the tangent plane at the origin, $T_{\mu_0}\mathcal{P}$:
                \[
                \mathbf{v}^k \sim \mathcal{N}(0, \Sigma)
                \]     
                \State \hspace{1em} (d) Parallel transport the tangent vector from the origin, $\mathbf{\mu}_0$, to the point $\bv_p$:
                \[
                \tilde{\bv}^k \gets PT_{\mathbf{\mu}_0 \rightarrow \bv_p}(\mathbf{v}^k) = (1 - \|\bv_p\|_2^2)\mathbf{v}^k
                \]
                \State \hspace{1em} (e) Apply the exponential map to the transported vector:
                \[
                        \dot{\bv}^k \gets \exp_{\bv_p}(\tilde{\bv}^k) \nonumber = \bv_p \oplus \left( \tanh \left( (1 - \|\bv_p\|_2^2) \|\tilde{\bv}^k\|_2 \right) \frac{\tilde{\bv}^k}{\|\tilde{\bv}^k\|_2} \right) 
                \]
                \State \hspace{1em} (f) Compute the left and right branch lengths via the geodesics:
                \[
                    \mathcal{B}_{L,r}^k \gets d(\bv_{L,r}, \dot \bv^k)\, , \,
                    \mathcal{B}_{R,r}^k \gets d(\bv_{R,r}, \dot \bv^k)
                \]
                \State \textsc{Compute weights:}
                \[
                w_r^k = \frac{\pi(s_r^k)}{\pi(s_{r-1}^{a_{r-1}^k})} \cdot \frac{\nu^{-}(s_{r-1}^{a_{r-1}^k})}{q(s_r^k | s_{r-1}^{a_{r-1}^k})}\cdot \text{det} \left( \dfrac{\partial d(\dot \bv^k, \bv_{i,r}^k)}{\partial \dot \bv^k}\right)
                \]
            \EndFor
        \EndFor
        \State \textbf{Output:} $s_{R}^{1:K}$ , $w_{1:R}^{1:K}$
    \end{algorithmic}
\end{algorithm}
The marginal likelihood estimator $\widehat{\mathcal{Z}}_{HCSMC}$ and the evidence lower bound (ELBO) $\mathcal{L}_{HCSMC}$ for \textsc{Vi}  are defined below:
\begin{equation}
    \widehat{\mathcal{Z}}_{HCSMC} \coloneqq \|\widehat{\pi}_{R}\| = \prod\limits_{r=1}^{R}\left(\frac{1}{K} \sum\limits_{k=1}^{K}w_{r}^k\right)\,, \qquad \mathcal{L}_{HCSMC} \coloneqq \mathbb{E} \left[\log \widehat{\mathcal{Z}}_{HCSMC} \right]
\end{equation}

\subsection{Hyperbolic Nested Combinatorial Sequential Monte Carlo}
\label{hyperbolicNCSMC}
\begin{algorithm}
    \caption{Nested Combinatorial Sequential Monte Carlo with Poincaré Embeddings}
    \begin{algorithmic}[1]
        \State \textbf{Input:} $\mathbf{Y} = \{Y_1,\cdots,Y_N \} \in \Omega^{N \times M}$,  $\theta = (\mathbf{Q},\{\lambda_i\}_{i=1}^{|E|})$
        \State \textbf{Embed sequences in $\mathcal{P}$:} $\forall n, Y_n \rightarrow f(Y_n)$ where $f: \Omega \rightarrow \mathcal{P}$
        \State \textbf{Initialization:} $\forall k$, $s_{0}^k \leftarrow \perp$, $w_{0}^k \leftarrow 1/K$
        \For{$r = 0$ \textbf{to} $R = N-1$}
            \For{$k = 1$ \textbf{to} $K$}
                \State \textsc{Resample: }
                \[
                \mathbb{P}(a_{r-1}^k = i) = \frac{w_{r-1}^i}{\sum_{l=1}^{K} w_{r-1}^l}
                \]
                \For{$j = 1$ to $L = \binom{N-r}{2}$}
                    \vspace{.5em}
                    \For{$m = 1$ to $M$}
                \State \textsc{Form Look-Ahead Partial State via Algorithm 1:}

                \State  \hspace{1em} Find the point on the geodesic arc between $\mathbf{v}_{L,r}^{k,m}[j]$ and  $\mathbf{v}_{R,r}^{k,m}[j]$ closest to the origin:
                \[
                \mathbf{v}_p^{k,m}[j] \gets \argmin_{\mathbf{v} \in \gamma(\mathbf{v}_{L,r}[j], \mathbf{v}_{R,r}[j])} d(\mu_0, \mathbf{v})
                \]
                \State \hspace{1em} Sample a Gaussian in $\mathbb{R}^2$ to yield a vector in the tangent plane at the origin, $T_{\mu_0}\mathcal{P}$:
                \[
                \mathbf{v}^{k,m}[j] \sim \mathcal{N}(0, \Sigma)
                \]     
                \State \hspace{1em} Parallel transport the tangent vector from the origin, $\mathbf{\mu}_0$, to the point $\mathbf{v}_p$:
                \[
                \tilde{\mathbf{v}}^{k,m}[j] \gets PT_{\mathbf{\mu}_0 \rightarrow \mathbf{v}_{p}[j]}(\mathbf{v}^{k,m}[j]) = (1 - \|\mathbf{v}_p^{k,m}[j]\|_2^2)\mathbf{v}^{k,m}[j]
                \]
                \State \hspace{1em} Apply the exponential map to the transported vector:
                \begin{align*}
                        \dot{\mathbf{v}}^{k,m}[j] &\gets \exp_{\mathbf{v}_p[j]}\nonumber 
                        (\tilde{\mathbf{v}}^{k,m}[j])\\ &= \mathbf{v}_p^{k,m}[j] \oplus \left( \tanh \left( (1 - \|\mathbf{v}_p^{k,m}[j]\|_2^2) \|\tilde{\mathbf{v}}^{k,m}[j]\|_2 \right) \frac{\tilde{\mathbf{v}}^{k,m}[j]}{\|\tilde{\mathbf{v}}^{k,m}[j]\|_2} \right)  \nonumber                  
                \end{align*}
                \State \hspace{1em} Compute the left and right branch lengths via the geodesics:
                \[
                    \mathcal{B}_{L,r}^{k,m}[j] \gets d(\mathbf{v}_{L,r}^{k,m}[j], \dot{\mathbf{v}}^{k,m}[j])\, , \, \mathcal{B}_{R,r}^{k,m}[j] \gets d(\mathbf{v}_{R,r}^{k,m}[j], \dot{\mathbf{v}}^{k,m}[j])
                \]
                \State \textsc{Compute Potentials (Sub-Weights)}
                \[ 
                    w_{r}^{k,m}[j] \coloneqq \dfrac{\pi(s_{r}^{k,m}[j])}{\pi(s_{r-1}^{a_{r-1}^k})}\cdot \dfrac{\nu^{-}(s_{r-1}^{a_{r-1}^k})}{q(s_{r}^{k,m}[j] |s_{r-1}^{a_{r-1}^k})}\,\cdot \, \text{det} \left( \dfrac{\partial \beta_{i,r}^{k,m}(\dot \bv^{k,m}[j])}{\partial \dot \bv^{k,m}[j]}\right)\,,
                \]
                \EndFor
                \EndFor
                \State \textsc{Extend Partial State}
                \[ s_r^k = s_{r}^{k,i}[J] \]
                \[ P(J=j, I = i) = \frac{w_{r}^{k,i}[j]}{\sum_{l=1}^{L}\sum_{m=1}^{M}w_r^{k,m}[j]} \]
                \State \textsc{Compute weights:}
                \[
                w_r^k = \frac{1}{ML}\sum\limits_{l=1}^{L}\sum\limits_{m=1}^{M}w_{r}^{k,i}[j]
                \]
                
            \EndFor
        \EndFor
        \State \textbf{Output:} $s_{R}^{1:K}$ , $w_{1:R}^{1:K}$
    \end{algorithmic}
    \label{Hncsmc_algorithm}
\end{algorithm}

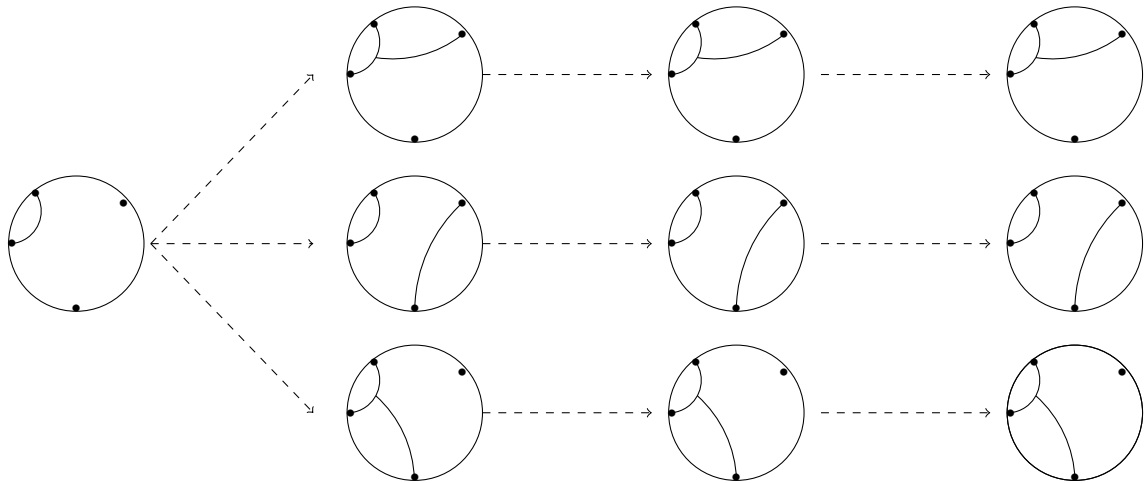
\begin{figure*}[h]
    \centering
    \begin{tikzpicture}[sloped, scale=0.9]

        \tikzset{
            dep/.style={
                circle, minimum size=1, fill=orange!20, draw=orange,
                general shadow={fill=gray!60, shadow xshift=1pt, shadow yshift=-1pt}
            },
            cli/.style={
                circle, minimum size=1, fill=white, draw,
                general shadow={fill=gray!60, shadow xshift=1pt, shadow yshift=-1pt}
            },
            obs/.style={
                circle, minimum size=1, fill=gray!20, draw,
                general shadow={fill=gray!60, shadow xshift=1pt, shadow yshift=-1pt}
            },
            spl/.style={
                cli=1, append after command={
                    node[circle, draw, dotted, minimum size=1.5cm] at (\tikzlastnode.center) {}
                }
            },
            c1/.style={-stealth, very thick, red!80!black},
            v2/.style={-stealth, very thick, yellow!65!black},
            v4/.style={-stealth, very thick, purple!70!black}
        }

        \pgfmathsetmacro{\rowAoffsetY}{4.25};
        \pgfmathsetmacro{\rowBoffsetY}{1.75};
        \pgfmathsetmacro{\rowCoffsetY}{-0.75};

        \pgfmathsetmacro{\colAoffsetX}{-6.5};
        \pgfmathsetmacro{\colBoffsetX}{-1.5};
        \pgfmathsetmacro{\colCoffsetX}{3.25};
        \pgfmathsetmacro{\colDoffsetX}{8.25};

        \pgfmathsetmacro{\xlocA}{-0.95};
        \pgfmathsetmacro{\xlocB}{0};
        \pgfmathsetmacro{\xlocC}{0.7};
        \pgfmathsetmacro{\xlocD}{-0.6};
        
        \pgfmathsetmacro{\ylocA}{0};
        \pgfmathsetmacro{\ylocB}{-0.95};
        \pgfmathsetmacro{\ylocC}{0.6};
        \pgfmathsetmacro{\ylocD}{0.75};
        \node (Resample) at (-4.5, 6) {\textsc{Enumerate Topologies}};
        \node (A) at (\colAoffsetX + \xlocA, \rowBoffsetY + \ylocA) {\tiny $\bullet$};
        \node (B) at (\colAoffsetX + \xlocB, \rowBoffsetY + \ylocB) {\tiny $\bullet$};
        \node (C) at (\colAoffsetX + \xlocC, \rowBoffsetY + \ylocC) {\tiny $\bullet$};
        \node (D) at (\colAoffsetX + \xlocD, \rowBoffsetY + \ylocD) {\tiny $\bullet$};

        \draw (D) arc[start angle=33.87, end angle=-83.90, radius=0.483]; 
        \draw (\colAoffsetX, \rowBoffsetY) circle [radius=1];
        \coordinate (bone) at (\colAoffsetX + 1.1, \rowBoffsetY);

        \node (Propose) at (0.75, 6) {\textsc{Subsample Branch Lengths}};

        \node (AA1) at (\colBoffsetX + \xlocA, \rowAoffsetY + \ylocA) {\tiny $\bullet$};
        \node (BB1) at (\colBoffsetX + \xlocB, \rowAoffsetY + \ylocB) {\tiny $\bullet$};
        \node (CC1) at (\colBoffsetX + \xlocC, \rowAoffsetY + \ylocC) {\tiny $\bullet$};
        \node (DD1) at (\colBoffsetX + \xlocD, \rowAoffsetY + \ylocD) {\tiny $\bullet$};
        
        \draw (DD1) arc[start angle=33.87, end angle=-83.90, radius=0.483]; 
         \node (μ_CC_DD) at (\colBoffsetX + -.575, \rowAoffsetY + 0.25) {};
         \draw (μ_CC_DD) arc[start angle=-98.67, end angle=-52.2, radius=1.65]; 
        
        \draw (\colBoffsetX, \rowAoffsetY) circle [radius=1];
        \coordinate (atwo) at (-3, \rowAoffsetY);
        \coordinate (athree) at (-0.5, \rowAoffsetY);

        \node (AA2) at (\colBoffsetX + \xlocA, \rowBoffsetY + \ylocA) {\tiny $\bullet$};
        \node (BB2) at (\colBoffsetX + \xlocB, \rowBoffsetY + \ylocB) {\tiny $\bullet$};
        \node (CC2) at (\colBoffsetX + \xlocC, \rowBoffsetY + \ylocC) {\tiny $\bullet$};
        \node (DD2) at (\colBoffsetX + \xlocD, \rowBoffsetY + \ylocD) {\tiny $\bullet$};

        \draw (DD2) arc[start angle=33.87, end angle=-83.90, radius=0.483]; 
        \draw (BB2) arc[start angle=178.65, end angle=132.74, radius=2.18];

        \draw (\colBoffsetX, \rowBoffsetY) circle [radius=1];

        \coordinate (btwo) at (-3, \rowBoffsetY);
        \coordinate (bthree) at (-0.5, \rowBoffsetY);

        \node (AA3) at (\colBoffsetX + \xlocA, \rowCoffsetY + \ylocA) {\tiny $\bullet$};
        \node (BB3) at (\colBoffsetX + \xlocB, \rowCoffsetY + \ylocB) {\tiny $\bullet$};
        \node (CC3) at (\colBoffsetX + \xlocC, \rowCoffsetY + \ylocC) {\tiny $\bullet$};
        \node (DD3) at (\colBoffsetX + \xlocD, \rowCoffsetY + \ylocD) {\tiny $\bullet$};
        
        \draw (DD3) arc[start angle=33.87, end angle=-83.90, radius=0.483]; 
        \node (μ_dd3_aa3) at (\colBoffsetX + -0.575, \rowCoffsetY + 0.25) {};]
        \draw (μ_dd3_aa3) arc[start angle=47.98, end angle=1.71, radius=1.72]; 
        \draw (\colBoffsetX, \rowCoffsetY) circle [radius=1];

        \coordinate (ctwo) at (-3, \rowCoffsetY);
        \coordinate (cthree) at (-0.5, \rowCoffsetY);

        \draw[->, dashed] (bone) -- (btwo);
        \draw[->, dashed] (bone) -- (atwo);
        \draw[->, dashed] (bone) -- (ctwo);

        \node (CAA1) at (\colCoffsetX + \xlocA, \rowAoffsetY + \ylocA) {\tiny $\bullet$};
        \node (CBB1) at (\colCoffsetX + \xlocB, \rowAoffsetY + \ylocB) {\tiny $\bullet$};
        \node (CCC1) at (\colCoffsetX + \xlocC, \rowAoffsetY + \ylocC) {\tiny $\bullet$};
        \node (CDD1) at (\colCoffsetX + \xlocD, \rowAoffsetY + \ylocD) {\tiny $\bullet$};
        
        \draw (CDD1) arc[start angle=33.87, end angle=-83.90, radius=0.483]; 
         \node (Cμ_CC_DD) at (\colCoffsetX + -.575, \rowAoffsetY + 0.25) {};
         \draw (Cμ_CC_DD) arc[start angle=-98.67, end angle=-52.2, radius=1.65]; 

        \draw (\colCoffsetX,\rowAoffsetY) circle [radius=1];

        \coordinate (afour) at (2,\rowAoffsetY);
        \draw[->, dashed] (athree) -- (afour);
        
        \node (AA3) at (\colCoffsetX + \xlocA, \rowBoffsetY + \ylocA) {\tiny $\bullet$};
        \node (BB3) at (\colCoffsetX + \xlocB, \rowBoffsetY + \ylocB) {\tiny $\bullet$};
        \node (CC3) at (\colCoffsetX + \xlocC, \rowBoffsetY + \ylocC) {\tiny $\bullet$};
        \node (DD3) at (\colCoffsetX + \xlocD, \rowBoffsetY + \ylocD) {\tiny $\bullet$};

        \draw (DD3) arc[start angle=33.87, end angle=-83.90, radius=0.483]; 
        \draw (BB3) arc[start angle=178.65, end angle=132.74, radius=2.18];
        \draw (\colCoffsetX,\rowBoffsetY) circle [radius=1];

        \coordinate (bfour) at (2, \rowBoffsetY);
        \draw[->, dashed] (bthree) -- (bfour);

        \node (CAA3) at (\colCoffsetX + \xlocA, \rowCoffsetY + \ylocA) {\tiny $\bullet$};
        \node (CBB3) at (\colCoffsetX + \xlocB, \rowCoffsetY + \ylocB) {\tiny $\bullet$};
        \node (CCC3) at (\colCoffsetX + \xlocC, \rowCoffsetY + \ylocC) {\tiny $\bullet$};
        \node (CDD3) at (\colCoffsetX + \xlocD, \rowCoffsetY + \ylocD) {\tiny $\bullet$};
        
        \draw (CDD3) arc[start angle=33.87, end angle=-83.90, radius=0.483]; 
        \node (Cμ_dd3_aa3) at (\colCoffsetX + -0.575, \rowCoffsetY + 0.25) {};]
        \draw (Cμ_dd3_aa3) arc[start angle=47.98, end angle=1.71,radius=1.75];
        \draw (\colCoffsetX,\rowCoffsetY) circle [radius=1];

        \coordinate (cfour) at (2, \rowCoffsetY);
        \draw[->, dashed] (cthree) -- (cfour);

        \node (Weighting) at (5.75, 6) {\textsc{Compute Potentials}};
        \node (DAA1) at (\colDoffsetX + \xlocA, \rowAoffsetY + \ylocA) {\tiny $\bullet$};
        \node (DBB1) at (\colDoffsetX + \xlocB, \rowAoffsetY + \ylocB) {\tiny $\bullet$};
        \node (DCC1) at (\colDoffsetX + \xlocC, \rowAoffsetY + \ylocC) {\tiny $\bullet$};
        \node (DDD1) at (\colDoffsetX + \xlocD, \rowAoffsetY + \ylocD) {\tiny $\bullet$};
        
        \draw (DDD1) arc[start angle=33.87, end angle=-83.90, radius=0.483]; 
         \node (Dμ_CC_DD) at (\colDoffsetX + -.575, \rowAoffsetY + 0.25) {};
         \draw (Dμ_CC_DD) arc[start angle=-98.67, end angle=-52.2, radius=1.65]; 
        \draw (\colDoffsetX,\rowCoffsetY) circle [radius=1];

        \draw (\colDoffsetX, \rowAoffsetY) circle [radius=1];

        \coordinate (afive) at (4.5, \rowAoffsetY);
        \coordinate (asix) at (7, \rowAoffsetY);
        \draw[->, dashed] (afive) -- (asix);

        \node (AA4) at (\colDoffsetX + \xlocA, \rowBoffsetY + \ylocA) {\tiny $\bullet$};
        \node (BB4) at (\colDoffsetX + \xlocB, \rowBoffsetY + \ylocB) {\tiny $\bullet$};
        \node (CC4) at (\colDoffsetX + \xlocC, \rowBoffsetY + \ylocC) {\tiny $\bullet$};
        \node (DD4) at (\colDoffsetX + \xlocD, \rowBoffsetY + \ylocD) {\tiny $\bullet$};

        \draw (DD4) arc[start angle=33.87, end angle=-83.90, radius=0.483]; 
        \draw (BB4) arc[start angle=178.65, end angle=132.74, radius=2.18];

        \draw (\colDoffsetX, \rowBoffsetY) circle [radius=1];

        \coordinate (bfive) at (4.5, \rowBoffsetY);
        \coordinate (bsix) at (7, \rowBoffsetY);
        \draw[->, dashed] (bfive) -- (bsix);

        \node (DAA3) at (\colDoffsetX + \xlocA, \rowCoffsetY + \ylocA) {\tiny $\bullet$};
        \node (DBB3) at (\colDoffsetX + \xlocB, \rowCoffsetY + \ylocB) {\tiny $\bullet$};
        \node (DCC3) at (\colDoffsetX + \xlocC, \rowCoffsetY + \ylocC) {\tiny $\bullet$};
        \node (DDD3) at (\colDoffsetX + \xlocD, \rowCoffsetY + \ylocD) {\tiny $\bullet$};
        
        \draw (DDD3) arc[start angle=33.87, end angle=-83.90, radius=0.483]; 
        \node (Dμ_dd3_aa3) at (\colDoffsetX + -0.575, \rowCoffsetY + 0.25) {};]
        \draw (Dμ_dd3_aa3) arc[start angle=47.98, end angle=1.71,radius=1.75];

        \draw (\colDoffsetX, \rowCoffsetY) circle [radius=1];

        \coordinate (cfive) at (4.5, \rowCoffsetY);
        \coordinate (csix) at (7, \rowCoffsetY);
        \draw[->, dashed] (cfive) -- (csix);
    \end{tikzpicture}
    \caption{Overview of the \textsc{Ncsmc} framework in the Poincaré disk. To extend a partial state for a single sample, we enumerate all $\binom{N-r}{2}$ possible lookahead topologies, subsample branch lengths by drawing $M$ points from a $\mathcal{WN}$ distribution, and compute sub-weights (potentials). One candidate is then sampled from these states based on its score. In the example shown, $N = 4$ leaf nodes are embedded, $r = 1$, yielding three topologies, and $M = 1$ Monte Carlo sample is used for branch lengths.}
\end{figure*}

We present an overview of the Hyperbolic Nested Sequential Monte Carlo (\textsc{H-Ncsmc}) framework, described in detail in Algorithm \ref{Hncsmc_algorithm}.

\textsc{H-Ncsmc} iterates through rank events (Step 4) and Monte Carlo samples (Step 5), performing a \textsc{Resample} step. For each sample, \textsc{H-Ncsmc} enumerates all $\binom{N-r}{2}$ possible one-step-ahead topologies, and samples corresponding $M$ sub-branch lengths using Algorithm 1 (Sampling Parent Node Embeddings in Hyperbolic Space). For each of the $j= 1, \dots, \binom{N-r}{2}$ one-step-ahead topologies, corresponding to Monte Carlo sample $k$, the two coalescing embedded points $\bv_{L,r}^k[j]$ and $\bv_{R,r}^k[j]$ represent the left and right child nodes. The closest point on the geodesic between $\bv_{L,r}^k[j]$ and $\bv_{R,r}^k[j]$ to the origin is identified (Step 10). The parent node is then sampled in the tangent plane at the origin (Step 11) and mapped via parallel transport to its mean (Step 12). Afterward, the exponential map is used to project the samples back onto the manifold (Step 13). The left and right branch lengths $\beta_{L,r}^{k,m}[j]$, $\beta_{R,r}^{k,m}[j]$ are computed using the geodesic distance between the sampled parent and the child nodes (Step 14). Each sampled look-ahead state $s_r^{k,m}[j]$ is assigned a sub-weight or potential function (Step 15). The partial state for $s_r^k$ is then extended by choosing one of the $\binom{N-r}{2}$ topologies and corresponding branch lengths based on their weights (Step 18). For each sample, we compute its weight by averaging over all sub-weights (Step 19).

In a slight abuse of notation, $i \in \{L,R\}$ refers to the left/right child nodes in Step 15, while $i \in \{1,\cdots,M\}$ refers to the sub-sampled parent embeddings in Steps 18 and 19. Although Algorithm \ref{Hncsmc_algorithm} is explicitly outlined, it is efficiently implemented in a fully vectorized form across $K$ and $M$. The marginal likelihood estimator $\widehat{\mathcal{Z}}_{H-NCSMC}$ and the evidence lower bound (ELBO) $\mathcal{L}_{H-NCSMC}$ for \textsc{Vi}  are defined below:

\begin{equation}
    \widehat{\mathcal{Z}}_{H-NCSMC} \coloneqq \|\widehat{\pi}_{R}\| = \prod\limits_{r=1}^{R}\left(\frac{1}{K} \sum\limits_{k=1}^{K}w_{r}^k\right)\,, \qquad \mathcal{L}_{H-NCSMC} \coloneqq \mathbb{E} \left[\log \widehat{\mathcal{Z}}_{H-NCSMC} \right]
\end{equation}
\clearpage

\begin{proposition}
    The particles $s_r^k$ and weights $w_r^k$ generated by Algorithm~\ref{Hncsmc_algorithm} are properly weighted for $\frac{\pi(s_r) \nu^-(s_{r-1})}{\pi(s_{r-1})}$.
\label{thm:pw2_}
\end{proposition}
\begin{proof}
\begin{align}
    \mathbb{E}[w_r^k h(s_r^k)] &= \mathbb{E}\left[w_r^k \cdot h(s_r^{k,I}[J])\right]   \nonumber \\
    &= \mathbb{E}\left[\sum_{j=1}^L \sum_{i=1}^M w_r^k \frac{w_r^{k,i}[j]}{\sum_l \sum_m w_r^{k,m}[l]} h(s_r^{k,i}[j])\right] \nonumber \\
    &= \frac{1}{ML} \sum_{j=1}^L \sum_{i=1}^M\mathbb{E}\left[w_r^{k,i}[j] \cdot h(s_r^{k,i}[j])\right] \nonumber \\
    &= \mathbb{E}\left[w_r^{k,i}[j] \cdot h(s_r^{k,i}[j])\right] \nonumber \\
  \intertext{Replacing index $i$ for index $m$ so that $i$ can be used to denote left and right branch lengths $\beta_i$:}
    &= \sum_{\tau} \int h( \dot{\bv}_r^{k,m}[j], \tau_r^{k}[j]) 
    \dfrac{\pi(\dot \bv_{r}^{k,m}[j], \tau_r^k)\cdot \nu^{-}(\dot \bv_{r-1}^{a_{r-1}^k}, \tau_{r-1}^{a_{r-1}^k})}{\pi(\dot \bv_{r-1}^{a_{r-1}^k}, \tau_{r-1}^{a_{r-1}^k})}\,\cdot \, \text{det} \left( \dfrac{\partial \beta_{i,r}^{k,m}(\dot \bv^{k,m}[j])}{\partial \dot \bv^{k,m}[j]}\right)
  d \dot{\bv}_r^{k,m}[j] 
  \nonumber \\
  &= \sum_{\tau} \int h(\beta_{i,r}^{k,m}[j], \tau_r^k[j]) 
  \dfrac{\pi(\beta_{i,r}^{k,m}[j], \tau_r^k) \nu^{-}(\beta_{i,r-1}^{a_{r-1}^k}, \tau_{r-1}^{a_{r-1}^k})}{\pi(\beta_{i,r-1}^{a_{r-1}^k}, \tau_{r-1}^{a_{r-1}^k})} 
  \, \mathrm{d}\beta_{i,r}^{k,m}[j] \nonumber\\
 &= \int h(s_r^k) \frac{\pi(s_r^k) \nu^-(s_{r-1}^{a_{r-1}^k})}{\pi(s_{r-1}^{a_{r-1}^k})} \, \mathrm{d}s_r^k \nonumber
\end{align}
\end{proof}

\subsection{Implementation Details}
\label{implementation_details}

The fast execution time of \textsc{H-Vcsmc} is a result of an efficient PyTorch implementation of the method. With the exception of the resampling step, are fully vectorized across particles and sites. The high degree of parallelization allows the method to fully saturate high-end \textsc{Gpu}s like the \textsc{Nvidia} A100.

Maximizing \textsc{Gpu} throughput also requires the method to work under \texttt{float32} precision. The primary numeric challenge is computing likelihoods via Felsenstein's pruning algorithm. Namely, each internal node is associated with a distribution $\mathbf{L} = (\exp(\mathbf Q \beta_\text{L})) \mathbf{L}_\text{L} \cdot (\exp(\mathbf Q \beta_\text{R})) \mathbf{L}_\text{R}$ over the alphabet, computed via the distributions ($\mathbf{L}_\text{L}, \mathbf{L}_\text{R}$) and branch lengths ($\beta_\text{L}, \beta_\text{R}$) of the left and right child. As a result of the recursive computation, $\mathbf L$ at the root node contains very small entries, leading to numeric instability in the forward and backward passes of the model. Performing these calculations in log space alleviates the numeric issues under \texttt{float32} precision.

We provide a modular PyTorch implementation of both \textsc{H-Vcsmc} and \textsc{H-Vncsmc} for comparative evaluation. The framework allows for various proposal distributions to be utilized. Additionally, we re-implemented \textsc{Vcsmc} and \textsc{Vncsmc} defining their proposal distributions and utilizing the same training code.

\subsubsection{Decoding Embeddings into Elements of the Rate Matrix}

An $A \times A$ rate matrix $\bQ$ can be factored into $A$ holding times $h_i$ and $A$ stationary probabilities $s_i$. A multi-layer perceptron maps embeddings to these $2A$ parameters. The softmax is used to obtain a valid stationary distribution, and the holding times are normalized so their mean is one. This normalization prevents the rate matrix from exploding and influencing the evolutionary time scale.

The embeddings are first feature-expanded so the rate matrix can depend more directly on an embedding's distance from the origin. In particular, embeddings are expanded into $D+1$ dimensions via
\begin{equation}
    \phi(p) = \qty(\norm{p}, \frac{p_1}{\norm{p}}, \frac{p_2}{\norm{p}}, \dots) \; .
\end{equation}

Constructing the full rate matrix involves first computing
\begin{equation}
\bQ' =
\bmqty{
    1/h_1  & \cdots & 1/h_1  \\
    \vdots & \ddots & \vdots \\
    1/h_A  & \cdots & 1/h_A
}
\bmqty{
    s_1 &        &     \\
        & \ddots &     \\
        &        & s_A \\
}
\; ,
\end{equation}
then taking $\bQ$ to be $\bQ'$, but with diagonal entries equal to the negative sum of each row's off-diagonal entries. By factorizing $\bQ$, the stationary probabilities required by Felsenstein's pruning algorithm are directly available.

\begin{figure}
    \centering
    \includegraphics[width=0.85\linewidth]{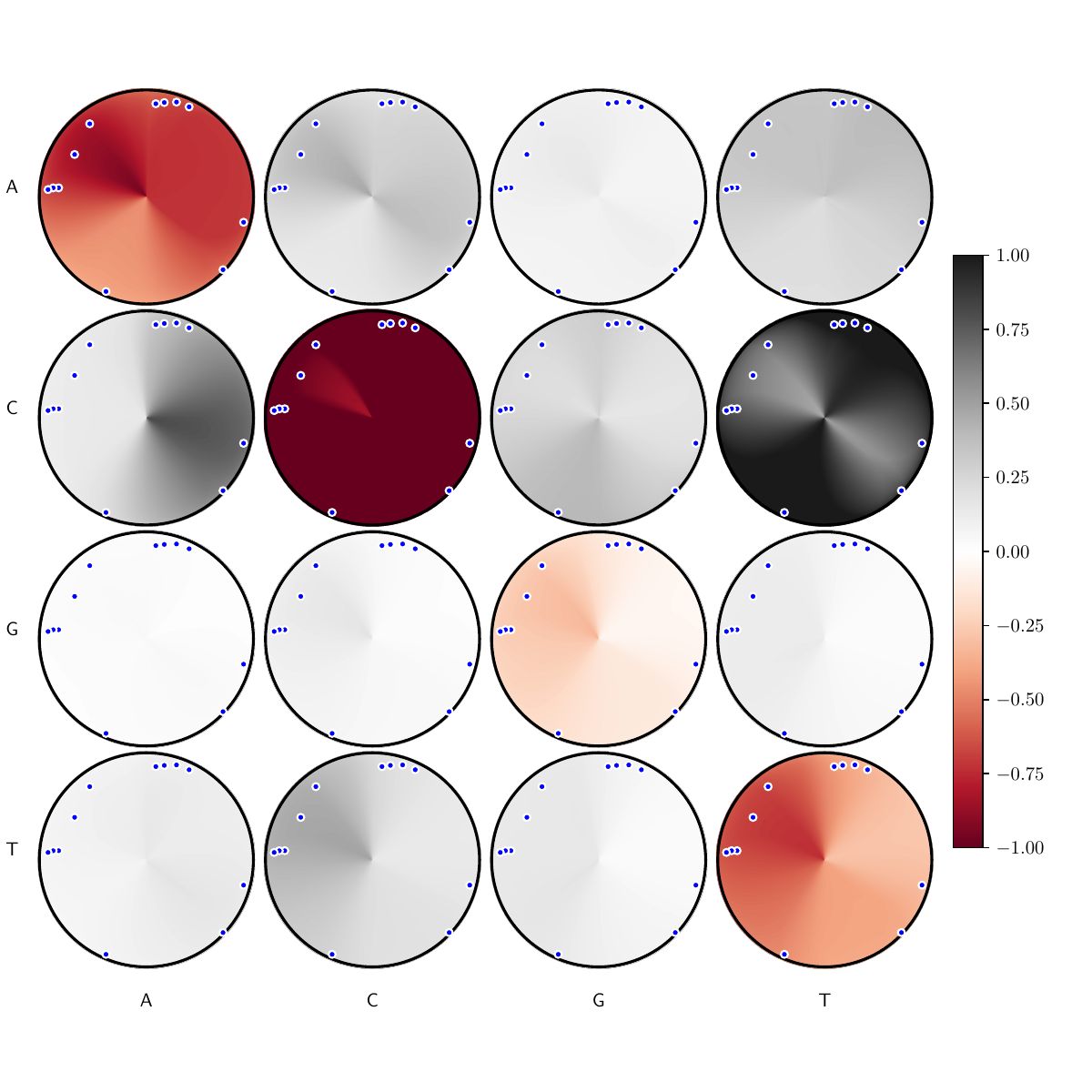}
    \caption{Visualization of the $Q$ matrix entries as a function of position in the Poincaré disk. The matrix entry in row $i \in \qty{A,G,C,T}$, column $j \in \qty{A,G,C,T}$ is constructed from a holding time and stationary probability: $Q_{ij} = (1/h_i) s_j$. The diagonals are negative such that the rows sum to zero. Embedded sequences from the Primates dataset are shown as blue points. The Q matrix varies smoothly as a function of radius and polar angle, suggesting that the learned function does not overfit.}
    \label{fig:q_matrix_parameter_space}
\end{figure}

\subsubsection{Deferring Branch Length Sampling}

A large share of the variance in the gradient estimate can be attributed to the randomness in sampling branch lengths in the proposal step (i.e., sampling the location of the parent nodes). This variance has an adverse effect on the convergence speed of \textsc{H-Vcsmc} and \textsc{H-Vncsmc}. A convenient feature of these two methods, however, is that they can fit model parameters governing the proposal distribution and rate matrix without sampling branch lengths. One can substitute in a modified proposal that takes the parent node embedding directly to be the point on the geodesic closest to the origin. Embeddings can be optimized first under this deterministic regime, then ``fine-tuned'' after switching back to the proper proposal that includes parent node sampling. Our experiments show that for larger datasets, performing these two training phases separately leads to better likelihood estimates.

\subsubsection{Memoizing Over Topologies}

Under the modified proposal that produces deterministic branch lengths, two forests $s_r^i$ and $s_r^j$ with matching topologies will have matching likelihoods $\pi(s_r^i) = \pi(s_r^j)$. Therefore, the computation of the $K$ important weights via Eq.~$\eqref{eq:weights}$ can be memoized for partial states with identical topologies, where the memoization key comes from a tree hashing algorithm that is invariant to child ordering. This memoization speeds up the forward pass and significantly reduces the size of the computation graph that must be saved in memory for computing gradients of Eq.~$\eqref{eq:smcmarginallikelihood}$.

\begin{figure}
\centering
\begin{tikzpicture}[scale=6]

    \draw[thick, black] (0,0) circle [radius=1];

    \draw[fill=blue, draw=black] (.7, -.46) circle (0.1pt) node[right] {\small Macaca\_fuscata};
    \draw[fill=blue, draw=black] (.67, -.53) circle (0.1pt) node[below] {\small M\_mulatta};
    \draw[fill=blue, draw=black] (.68, -.48) circle (0.1pt);
    
    \draw (.7, -.46) arc[start angle=130.62, end angle=139.37, radius=.185];
    \draw (.67, -.53) arc[start angle=176.31, end angle=161.06, radius=.192];

    \draw[fill=blue, draw=black] (.8, -.39) circle (0.1pt) node[right] {\small M\_fascicularis};
    \draw[fill=blue, draw=black] (.7, -.42) circle (0.1pt);

    \draw (.68, -.48) arc[start angle=170.45, end angle=152.68, radius=.204];
    \draw (.8, -.39) arc[start angle=93.85, end angle=119.54, radius=.235];

    \draw[fill=blue, draw=black] (0.84, -0.31) circle (0.1pt) node[right] {\small M\_sylvanus};
    \draw[fill=blue, draw=black] (.72, -.34) circle (0.1pt);

    \draw (0.84, -0.31) arc[start angle=91.88, end angle=116.19, radius=.294];
    \draw (.7, -.42) arc[start angle=176.297, end angle=155.63, radius=.2298];

    \draw[fill=blue, draw=black] (.3, -.76) circle (0.1pt) node[below right] {\small Humans};
    \draw[fill=blue, draw=black] (.11, -.85) circle (0.1pt) node[below right] {\small Pan};
    \draw[fill=blue, draw=black] (0.18, -0.76) circle (0.1pt);

    \draw (.3, -.76) arc[start angle=76.09, end angle=103.9, radius=.2497];
    \draw (.11, -.85) arc[start angle=153.85, end angle=130.4, radius=.28];

    \draw[fill=blue, draw=black] (0.08, -0.89) circle (0.1pt) node[below left] {\small Gorilla};
    \draw[fill=blue, draw=black] (-0.04, -0.71) circle (0.1pt);

    \draw (0.18, -0.76) arc[start angle=84.55, end angle=95.44, radius=.2636];
    \draw (0.08, -0.89) arc[start angle=168.90, end angle=149.02, radius=.403];

    \draw[fill=blue, draw=black] (-0.23, -0.85) circle (0.1pt) node[right] {\small Pongo};
    \draw[fill=blue, draw=black] (-.14, -.68) circle (0.1pt);

    \draw (-0.04, -0.71) arc[start angle=107.92, end angle=144.84, radius=.37];
    \draw (-0.04, -0.71) arc[start angle=65.26, end angle=81.33, radius=0.3735];

    \draw[fill=blue, draw=black] (-0.31, -0.87) circle (0.1pt) node[above left] {\small Hylobates};
    \draw[fill=blue, draw=black] (.13, -.76) circle (0.1pt);

    \draw (-0.04, -0.71) arc[start angle=88.37, end angle=58.84, radius=0.3476];
    \draw (-0.31, -0.87) arc[start angle=151.88, end angle=124.47, radius=0.538];

    \draw[fill=blue, draw=black] (.2, -.39) circle (0.1pt);
    \draw (-.14, -.68) arc[start angle=144.49, end angle=116.433, radius=.922];
    \draw (.72, -.34) arc[start angle=79.1, end angle=111.88, radius=.926];

    \draw[fill=blue, draw=black] (-0.78, -0.48) circle (0.1pt) node[below right] {\small Saimiri\_sciureus};
    \draw[fill=blue, draw=black] (-0.08, -0.28) circle (0.1pt);

    \draw (-0.78, -0.48) arc[start angle=118.577, end angle=93.31, radius=1.66];
    \draw (.2, -.39) arc[start angle=63.059, end angle=74.045, radius=1.57];

    \draw[fill=blue, draw=black] (-0.00, 0.80) circle (0.1pt) node[above] {\small Lemur\_catta};
    \draw[fill=blue, draw=black] (0.20, 0.87) circle (0.1pt) node[below right] {\small Tarsius\_syrichta};

    \draw[fill=blue, draw=black] (-0.04, 0.01) circle (0.1pt);
    \draw (-0.04, 0.01) arc[start angle=-7.8275, end angle=-9.027, radius=13.44];
    \draw (-0.04, 0.01) arc[start angle = -4.746, end angle = -1.05,radius = 12.26];

    \draw[fill=blue, draw=black] (0,0) circle (0.1pt);
    \draw (0,0) -- (-0.04, 0.01);
    \draw (0,0) -- (0.20, 0.87);

\end{tikzpicture}
\caption{The maximum likelihood phylogeny inferred by \textsc{H-Vcsmc} on the Primates dataset, with a topology consistent with that recovered by MrBayes. Monkeys, hominids, and prosimians are partitioned into distinct clades, each occupying separate regions of the Poincaré disk.}
\end{figure}
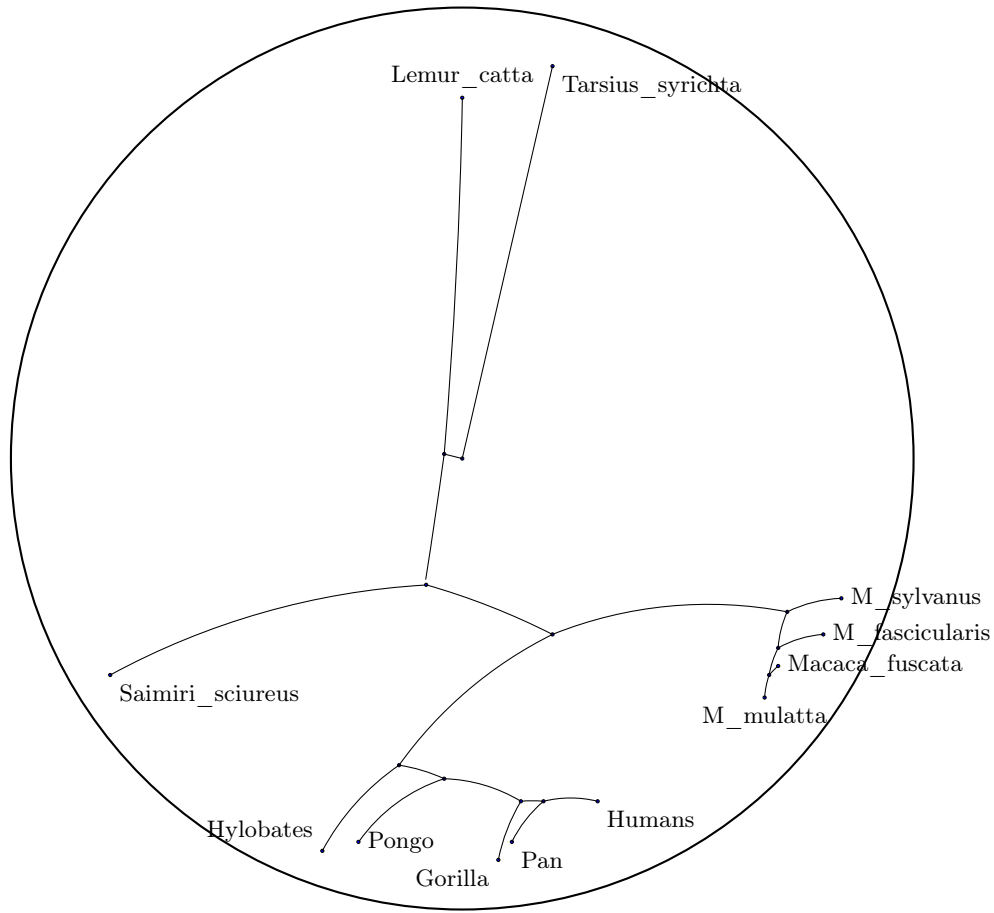

\end{document}